\documentclass[11pt]{article}

\usepackage{xcolor}
\usepackage{fancyhdr}
\usepackage{graphicx}
\usepackage[colorlinks=true, citecolor=blue, linkcolor=blue]{hyperref}
\usepackage{amsmath}
\usepackage{amsthm}  
\usepackage{amssymb}
\usepackage{commath}
\usepackage{color}
\usepackage{mathtools} 
\usepackage{multirow}
\usepackage{graphicx}
\usepackage[labelsep=colon]{caption}
\usepackage{algorithm}
\usepackage{algpseudocode}
\usepackage{booktabs}
\usepackage{xcolor}
\usepackage{adjustbox}
\usepackage{subcaption}
\usepackage{tabularx}
\usepackage{float}
\usepackage{svg}
\usepackage{wrapfig}
\usepackage{titlesec}
\usepackage{siunitx}

\usepackage{pgf,tikz}
\usetikzlibrary{arrows}

\usepackage{nicematrix}
\usepackage{arydshln}
\NiceMatrixOptions
{
	custom-line =  
	{
		command = hdashedline , 
		letter = I , 
		tikz = dashed ,
		total-width = \pgflinewidth 
	}
}
\usepackage[shortlabels]{enumitem}
\setlist{leftmargin=*}
\newcommand{\enumlabel}[1]{\textbf{#1.\arabic*}}
\newcommand{\enumref}[1]{{#1.\arabic*}}

\algnewcommand{\algorithmicbreak}{\textbf{break}}

\algnewcommand{\algorithmicor}{ \textbf{or} }

\usepackage{cleveref}
\crefname{section}{\S}{\S\S}
\Crefname{section}{\S}{\S\S}

\titlelabel{\thetitle.\quad}

\makeatletter
\renewcommand{\maketitle}{\bgroup\setlength{\parindent}{0pt}
	\begin{flushleft}
		\textbf{\LARGE \@title}\vspace*{1.5em}
		
		\textbf{\@author}
	\end{flushleft}\egroup
}

\fancypagestyle{firstpagestyle}{
	\fancyhf{}
	\fancyhead[R]{\textbf{Preprint}}
}

\if \thepage=1 {\chead{}} \else  \fi
\makeatother

\newtheorem{definition}{Definition}
\newtheorem{lemma}{Lemma}
\newtheorem{proposition}{Proposition}
\newtheorem{theorem}{Theorem}
\newtheorem{remark}{Remark}

\newcounter{savefootnote}
\newcounter{symfootnote}
\newcommand{\symfootnote}[1]{%
	\setcounter{savefootnote}{\value{footnote}}%
	\setcounter{footnote}{\value{symfootnote}}%
	\ifnum\value{footnote}>8\setcounter{footnote}{0}\fi%
	\let\oldthefootnote=\thefootnote%
	\renewcommand{\thefootnote}{\fnsymbol{footnote}}%
	\footnote{#1}%
	\let\thefootnote=\oldthefootnote%
	\setcounter{symfootnote}{\value{footnote}}%
	\setcounter{footnote}{\value{savefootnote}}%
}

\newcommand{\calB}{{\mathcal{B}}}
\newcommand{\calC}{{\mathcal{C}}}
\newcommand{\calE}{{\mathcal{E}}}

\newcommand{\calX}{{\mathcal{X}}}

\newcommand{\calS}{{\mathcal{S}}}

\newcommand{\calL}{{\mathcal{L}}}

\allowdisplaybreaks
\newcommand{\rr}{{\mathbb R}}
\newcommand{\rnn}{{\mathbb R_{\ge0}}}
\newcommand{\rp}{{\mathbb R}_{>0}}

\newcommand{\cc}[1]{{\mathcal C^{#1}}}

\newcommand{\st}{\mathop{\, | \,}\nolimits}

\newcommand{\argmin}{\mathop{\rm argmin}}

\newcommand{\sign}{\mathop{\rm sgn}\nolimits}
\newcommand{\tr}{\mathop{\rm tr}\nolimits}

\newcommand{\dom}{\mathop{\rm dom}\nolimits}

\newcommand{\epilim}{\mathop{\rm \textrm{e--}\lim}\nolimits}

\newcommand{\gradn}{\mathop{\nabla}\nolimits}
\newcommand{\hessn}{\mathop{\nabla^2}\nolimits}
\newcommand{\thirdn}{\mathop{\nabla^3}\nolimits}
\newcommand{\extended}[1]{\mathop{#1\colon\rr^p\to \rr \cup \{+\infty\}}\nolimits}

\newcommand{\infconv}[2]{\mathop{#1 \square #2}\nolimits}

\newcommand{\pclsc}[1]{\mathop{\Gamma_0(#1)}\nolimits}
\newcommand{\cf}{{cf.}}
\newcommand{\eg}{{e.g.}}
\newcommand{\ie}{{i.e.}}

\newcommand{\Jaug}{{\hat J}}
\newcommand{\Qaug}{{\hat Q}}
\newcommand{\eaug}{{\hat e}}
\newcommand{\Gaug}{{\hat G}}
\newcommand{\smin}[1]{{\sigma_{\min}(#1)}}
\newcommand{\smax}[1]{{\sigma_{\max}(#1)}}
\def\RF{{\rm RF}}

\makeatletter
\newenvironment{proof*}[1][\proofname]{\par
	\pushQED{\qed}%
	\normalfont \partopsep=\z@skip \topsep=\z@skip
	\trivlist
	\item[\hskip\labelsep
	\itshape
	#1\@addpunct{.}]\ignorespaces
}{%
	\popQED\endtrivlist\@endpefalse
}

\newcommand{\tPhi}{{\tilde{\Phi}}}
\newcommand{\tG}{{\tilde{G}}}

\usepackage[numbers]{natbib}

\title{Regularized Gauss-Newton for Optimizing Overparameterized Neural Networks\par}

\author
{Adeyemi D. Adeoye~$^{1}$\symfootnote{Email address: adeyemi.adeoye@imtlucca.it}, Philipp Petersen~$^{2}$, Alberto Bemporad~$^{1}$
\\
\vspace{1em}
\normalfont{\small $^{1}$IMT School for Advanced Studies Lucca, Italy\\
$^{2}$Faculty of Mathematics, University of Vienna, Austria}\\
}

\begin{document}
\setcounter{footnote}{0}
\maketitle
\thispagestyle{firstpagestyle}

\begin{abstract}
The generalized Gauss-Newton (GGN) optimization method incorporates curvature estimates into its solution steps, and provides a good approximation to the Newton method for large-scale optimization problems. GGN has been found particularly interesting for practical training of deep neural networks, not only for its impressive convergence speed, but also for its close relation with neural tangent kernel regression, which is central to recent studies that aim to understand the optimization and generalization properties of neural networks. This work studies a GGN method for optimizing a two-layer neural network with explicit regularization. In particular, we consider a class of generalized self-concordant (GSC) functions that provide smooth approximations to commonly-used penalty terms in the objective function of the optimization problem. This approach provides an adaptive learning rate selection technique that requires little to no tuning for optimal performance. We study the convergence of the two-layer neural network, considered to be overparameterized, in the optimization loop of the resulting GGN method for a given scaling of the network parameters. Our numerical experiments highlight specific aspects of GSC regularization that help to improve generalization of the optimized neural network. The code to reproduce the experimental results is available at \url{https://github.com/adeyemiadeoye/ggn-score-nn}.
\end{abstract}

\section{Introduction}
Despite their superior convergence rates compared to first-order methods, (approximate) second-order methods are still rarely used --- and as such, underexplored --- for training large-scale machine learning and neural network (NN) models. This is due to their highly prohibitive computations and memory footprints at each iteration. Some past and recent works have, however, made efforts to reduce this overhead by proposing different approximations to the Hessian of the loss function, which the methods ultimately exploit to achieve their impressive convergence properties (see \eg, \cite{roux2007topmoumoute, martens2010deep, vinyals2012krylov, martens2015optimizing, botev2017practical, arbel2019kernelized, ren2019efficient, cai2019gram, adeoye2023score}).

One of the most appealing approximations to the Hessian matrix within the context of practical deep learning and nonlinear optimization in general is the generalized Gauss-Newton (GGN) approximation of \cite{schraudolph2002fast}, which uses a positive semi-definite (PSD) matrix to model the curvature about an arbitrary convex loss function. In fact, the Fisher information matrix (FIM) --- a curvature approximating matrix which most other approximate second-order methods seek to estimate --- is shown to have direct connections with the GGN matrix in many practical cases \citep{martens2015optimizing, pascanu2013revisiting}. Despite its close connection with the GGN matrix, the FIM, unlike the GGN matrix, potentially leads to over-approximating the second-order terms in more general loss functions, throwing away relevant curvature information \citep{schraudolph2002fast}. In addition to the desirable property of maintaining positive-definiteness throughout the training procedure, other nice properties of the GGN matrix, in comparison with the Hessian matrix, are discussed in \cite[Section 8.1]{martens2020new}; see also \cite{chen2011hessian} for discussions in the context of nonlinear least-squares estimation and \cite{bemporad2023training} for efficient training of (deep) recurrent neural networks with a GGN approach.

Towards understanding the theoretical working of deep neural networks, a line of work \citep{daniely2017sgd, li2018learning, du2018gradient, du2019gradient, allen2019convergence, zou2018stochastic, jacot2018neural, arora2019fine, chizat2019lazy, yang2019scaling} attributes their optimization and generalization success in many applications to their immense overparameterization, that is, the property of having way more parameters than the number of data points they are being trained on. These generalization properties of the NN are known to have connections with the implicit regularization of the overparameterized NN by the gradient descent (GD) method \citep{zhang2021understanding, li2018algorithmic, gunasekar2018characterizing, ji2019implicit, allen2019learning, chizat2020implicit, li2021happens, tarmoun2021understanding, berner2021modern}. For the (generalized) Gauss-Newton and its related FIM (or \emph{natural gradient}), some recent works \citep{cai2019gram, zhang2019fast, karakida2020understanding, kerekes2021depth, garcia2022fisher, arbel2023rethinking} have shown similar approximation properties and global convergence in the overparameterized regime, also mostly attributing generalization to the implicit regularization effect of the Gauss-Newton via the NTK \citep{cai2019gram, zhang2019fast, karakida2020understanding, kerekes2021depth, garcia2022fisher} and the mean-field \citep{arbel2023rethinking}.

In many of the works showing the implicit regularization effect of gradient-based optimizers, it is suggested that explicit regularizers are not needed at all in order to see the impressive generalization results. However, recent works such as \cite{wei2019regularization,raj2021explicit,orvieto2023explicit} argue that explicit regularization of the network indeed matters and should at least be given as much attention, both from a generalization and an optimization point of view. In particular, \cite{wei2019regularization} proved an approximation bound (in number of samples), via the lens of \emph{margin theory}, for an infinite-width one-hidden-layer NN \emph{weakly-regularized} by the $\ell_2$-norm, which significantly improves upon other results that rely on the NTK and/or implicit regularization formalism. In addition, they proved a global polynomial convergence rate for the noisy gradient descent, an improvement over related works that similarly study NN optimization in the infinite-width limit. In \cite{cai2019gram}, the interpretation of the GGN updates as an explicit solution of the NTK regression is used to prove a global linear convergence in the mini-batch setting. Apart from the explicit addition of a regularization term to the objective function, explicit regularization is also induced in other forms \cite{hanson1988comparing, yao2007early, srivastava2014dropout, rudi2015less, raj2021explicit, orvieto2023explicit}.

In this paper, we study the optimization of a one-hidden-layer NN by the GGN method, and by drawing inspiration from their performance in convex optimization, we consider explicit self-concordant regularization of the GGN. To the best of our knowledge, our convergence result is the first in this kind of setting: optimization of an explicitly regularized NN by the GGN method in the overparameterized regime. The structure of the class of regularization functions considered not only helps to control the local rate of change of their second derivatives \citep{owen2013self}, but is also used in the selection of adaptive learning rates. Unlike \cite{wei2019regularization}, we do not assume an arbitrarily \emph{weak} regularization under our setting; instead the smoothing framework covered by our study allows to choose a regularization strength which may depend only on the initialization of the NN, and is characterized by a smoothing parameter. However, for a proper choice of the regularization strength, the final trained NN model can be made to be reasonably small and ``simple" in spite of the overparameterization, and we can have a tradeoff between test error and training error.

\subsection{Notation} The standard Euclidean norm is denoted by $\norm{\cdot}$ or $\norm{\cdot}_2$ and the $1$-norm by $\norm{\cdot}_1$. We denote the standard inner product between two vectors by $\left<\cdot,\cdot\right>$, \ie,  $\left<x,y\right> \triangleq x^\top y$ for $x,y\in\rr^p$. For a positive integer $m$, we define $[m]\triangleq \{1,2,\ldots,m\}$. We let $\rnn$ and $\rp$ denote the set of nonnegative and positive real numbers, respectively. For an extended real-valued function $\extended{g}$, we denote by $\dom{g} \triangleq \{x\in \rr^p \st g(x) < +\infty \}$ the \emph{(effective) domain} of $g$. $\pclsc{\calX}$ denotes the set of proper convex lower-semicontinuous (lsc) functions from $\calX \subseteq \rr^p$ to $\rr \cup \{+\infty\}$. We denote by $\calC^k(\rr^p)$, the class of $k$-times continuously-differentiable functions on $\rr^p$, $k\in\rnn$. For $g\in \cc{3}(\dom{g})$, we let $g'(t)$, $g''(t)$ and $g'''(t)$ denote the first, second, and third derivatives of $g$, at $t\in \rr$, respectively. The gradient, Hessian, and third-order derivative tensor of $g$ at $x\in \rr^p$ are respectively written as $\gradn _x g(x)$, $\hessn _x g(x)$, and $\thirdn _x g(x)$. We omit the subscripts if the variables with respect to which the derivatives are taken are clear from the context. For a symmetric matrix $H\in\rr^{p\times p}$, we write $H\succ 0$ (resp. $H\succeq 0$) to say $H$ is positive definite (resp., positive semidefinite). We let $\lambda_1(H)$ denote the maximum eigenvalue of a matrix $H\in\rr^{p\times p}$, and $\lambda_p(H)$ its minimum eigenvalue; $\tr(H)$ denotes the trace of $H$. The scalars $\smax{A}$ and $\smin{A}$ respectively denote the maximum and minimum singular values of an $m\times p$ matrix $A$. Given that $\hessn g(x) \succ 0$, the \emph{local} norm $\norm{\cdot}_x$ with respect to $g$ at $x$ is the weighted norm induced by $\hessn g(x)$, \ie, $\norm{d}_x\triangleq \left<\hessn g(x)d,d\right>^{1/2}$. The dual norm is $\norm{v}_x^*\triangleq \left<\hessn g(x)^{-1}v,v\right>^{1/2}$. We also define the notations $\norm{x}_H \triangleq \left<Hx, x\right>^{\frac{1}{2}}$, $\norm{x}_H^* \triangleq \left<H^{-1}x, x\right>^{\frac{1}{2}}$, for $H \succ 0, x\in \rr^p$. An Euclidean ball of radius $r$ centered at $\bar{x}$ is denoted by $\calB_r(\bar{x}) \triangleq \{x \in \rr^p \st \norm{x - \bar{x}} \le r \}$. The (Dikin) ellipsoid of radius $r$ centered at $\bar{x}$ is defined by $\calE_r(\bar{x})\triangleq \{x\in \rr^p \st \norm{x-\bar{x}}_H < r\}$, for $H \succ 0$. We define set convergence in the sense of Painlev\'{e}-Kuratowski \cite[Chapter 4]{rockafellar2009variational}. Given $\{g_t\}_{k\in\rnn}$ with $g_t\colon \rr^p \to \rr \cup \{-\infty, +\infty\}$, $\epilim g_t=g$ denotes the epigraphic convergence (\emph{epi-convergence}) of $\{g_t\}_{k\in\rnn}$ to a function $g\colon \rr^p \to \rr \cup \{-\infty, +\infty\}$.
\section{GGN for Learning Neural Networks}\label{ss:ggn-settings}
Given the sequence of data points $S \triangleq \{(x_i,y_i)\}_{i\in [m]}$ with $x_i \in \rr^{n_0}, y_i\in \rr^{n_L}$, an $L$-layer fully-connected feedforward NN is defined as follows. Starting with an input $z_0\in \rr^{n_0\times m}$, and for $l=1,\ldots,L$,
\begin{align}
	a_l = W^{(l)} z_{l-1} + b^{(l)}, \quad z_l = \varrho_l(a_l),
\end{align}
where $W^{(l)}\in \rr^{n_l \times n_{l-1}}$ and $b^{(l)} \in \rr^{n_l}$ are the $l$-th layer weights and biases of the network, respectively; each $\varrho_l\colon \rr \to \rr$ is an element-wise activation function. Let $\Phi(\cdot;\theta) \eqqcolon z_L$ be the output of the NN, where $\theta = [\theta_1,\theta_2, \ldots, \theta_L]^\top \in \rr^p$ with $\theta_l \triangleq vec([W^{(l)}~b^{(l)}])$, the stacked vectorization of $W^{(l)}$ and $b^{(l)}$. In the \emph{supervised learning} task, we look for the parameter vector $\theta$ minimizing the regularized empirical risk
\begin{align}\label{eq:reg-prob}
	\min\limits_{\theta \in \rr^p} \calL(\theta) \triangleq \hat{R}_s(\Phi) + g(\theta), \qquad \hat{R}_s(\Phi) \triangleq \frac{1}{m}\sum_{i=1}^{m}\ell(\Phi(x_i;\theta),y_i),
\end{align}
where $\hat{R}_s(\Phi)$ is the empirical risk associated with the NN learning task, $\ell \colon \rr^{n_L} \times \rr^{n_L} \to \rr$ is a loss function, and $g\colon \rr^p \to \rr$ is a regularization function. We denote by $\Phi^*$ an output function that best interpolates the data set $S$.

Let $n_1 \equiv n$ (number of hidden neurons), $W^{(1)}\equiv u=[u_1, u_2,\ldots,u_n]^\top\in\rr^{n\times n_0}$ and $W^{(2)}\equiv v=[v_1, v_2,\ldots,v_n]\in\rr^n$. Without loss of generality, we consider a \emph{biasless} one-hidden layer NN:
\begin{align}\label{eq:nn}
	\rr^{n_0} \ni x \mapsto \Phi(x;\theta) \triangleq \kappa(n)\sum_{i=1}^{n} v_i\varrho(u_i x),
\end{align}
where $\kappa(n)$ is some scaling that depends on $n$, \eg, $\kappa(n) = 1/\sqrt{n}$ as in \cite{du2018gradient}. We remark that for an $(L-1)$-hidden layer NN written in the biasless form, the bias vectors can always be recovered by redefining
\begin{align*}
	x \leftarrow
	\begin{bmatrix}
		x\\1
	\end{bmatrix}, \quad
	\theta_l \leftarrow
	vec\left(
	\begin{bmatrix}
		W^{(l)} & b^{(l)}\\
		0 & 1
	\end{bmatrix}
	\right),
\end{align*}
for $l\in[L-1]$, and $\theta_L \leftarrow vec([W^{(L)}~b^{(L)}])$. We assume the following about the activation function $\varrho$, which is satisfied by most activation functions but piecewise linear ones.
\begin{enumerate}[label=\enumlabel{A}, ref=\enumref{A}]
	\item The activation function $\varrho$ is twice differentiable, Lipschitz, and smooth.\label{ass:a}
\end{enumerate}
Below, we briefly describe the NTK regression and its connection with GGN for overparameterized networks by first considering the case $g=0$ in \eqref{eq:reg-prob}.
\paragraph{The NTK and Gradient Descent.}
In the infinite-width limit, there is an established \citep{jacot2018neural} relation between the steps obtained via a gradient-based method for NNs and the so-called \emph{kernel gradient descent} in function space. In particular, it is shown that, as $n\to \infty$, $\forall i,j\in[m]$, $\langle\nabla_\theta \Phi(x_i,\theta_0), \nabla_\theta\Phi(x_j,\theta_0)\rangle$ converges to some positive definite deterministic kernel $k(x_i,x_j)=k(x_i,x_j)^\top\in \rr^{n_L\times n_L}$ (the limiting NTK), and remains unchanged during training. Consider the case $g=0$ in \eqref{eq:reg-prob}. In the infinite-width limit, the gradient descent for solving the resulting problem reduces to the kernel gradient descent:
\begin{align}
	\Phi_{t+1} = \Phi_t - \alpha_t G_t\nabla_{\Phi_t} \hat{R}_s(\Phi_t),\label{eq:kgd}
\end{align}
where $\Phi_t\triangleq(\Phi(x_i;\theta_t))_{i\in [m]}\in\rr^{m}$ denotes the network outputs on $x_i$'s at iteration $t$, $\alpha_t \in\rp$ is a step-size (or \emph{learning rate}), and $G_t$ is an $m\times m$ matrix whose $(i,j)$-th entry is given by $\langle\nabla_\theta \Phi(x_i,\theta_t), \nabla_\theta\Phi(x_j,\theta_t)\rangle$; see, \eg, \cite[Lemma 3.1]{arora2019exact} which considers the continuous-time evaluation of $\Phi_t$, $t\in \rnn$.

\paragraph{GGN and the NTK regression.}
An important feature of GGN for infinite-width NNs is its direct relation with the NTK regression solution in the overparameterized regime. We introduce the notations $Q_t \equiv \nabla_{\Phi_t}^2\hat{R}_s(\Phi_t)$, $e_t \equiv \nabla_{\Phi_t}\hat{R}_s(\Phi_t)$, and consider again the case $g=0$ in problem \eqref{eq:reg-prob}. The GGN for the resulting problem is given by the following iterative process:
\begin{align}
	\theta_{t+1} = \theta_t - \alpha_t(J_t^\top Q_t J_t)^{-1}J_t^\top e_t,\label{eq:ggn}
\end{align}
where $J_t = (\nabla_\theta\Phi(x_1,\theta_t), \ldots, \nabla_\theta\Phi(x_m,\theta_t))^\top\in\rr^{m\times p}$ is the Jacobian matrix (of features) at iteration $t$. For overparameterized networks, if $\ell$ is the squared loss, $Q_t$ becomes the identity matrix and we can conveniently rewrite the GGN updates with respect to the NTK matrix $G_t$ as
\begin{align}
	\theta_{t+1} = \theta_t - \alpha_tJ_t^\top G_t^{-1}e_t.\label{eq:ggn-ntk}
\end{align}
On the other hand, corresponding to \eqref{eq:kgd}, the updates to the parameters $\theta$ via $\theta_t$ can be obtained by solving the regression problem (\cf~\cite{cai2019gram})
\begin{align}\label{eq:ntk-prob}
	\theta_{t+1} = \argmin_{\theta}\frac{1}{2}\norm{\langle J_t,\theta - \theta_t\rangle + \nabla_{\Phi_t}\hat{R}_s(\Phi_t)}^2,
\end{align}
which results from the linearization of $\Phi$ around $\theta_t$. For an overparameterized NN, this linearization provides a good approximation to $\Phi$, and hence the Hessian (with respect to $\theta$) of the resulting empirical risk by replacing $\Phi$ by its linearization, which gives the GGN approximation (see, \eg, \cite{martens2011learning}), is expected to provide a good approximation of the actual Hessian. In terms of kernel ``ridgeless" regression solution to problems of the form \eqref{eq:reg-prob} (with $g=0$) and $\ell$ taken as the squared loss, the update \eqref{eq:ggn-ntk} provides a \emph{minimum-norm} interpolating solution in the so-called Reproducing Kernel Hilbert Space (RKHS) \citep{liang2020just, mohri2018foundations}. Hence, the GGN, in this case, provides a closed-form solution to the NTK regression, which efficiently replaces gradient descent in the NTK formalism.

\subsection{Regularized GGN for Overparameterized Neural Networks}
If $g\ne 0$, the relation between gradient descent and NTK will probably break \citep{wei2019regularization}. Apart from the NTK parameterization, a commonly studied parameterization in the context of overparameterized NNs is the random feature (RF) model \citep{rahimi2007random, balcan2006kernels} which, due to its close connection with a one-hidden layer NN, often provides a prototype for studying realistic NNs. Whether the NTK or the RF parameterization is used, one of the key properties we desire about the dynamics of the optimizer is \emph{stability} which, when established, can help to still benefit a lot from overparameterization in the optimization scope. To this end, we first present the definition of generalized self-concordant (GSC) functions on $\rr^p$ from \cite{sun2019generalized} as follows.
\begin{definition}\label{def:gsc-rn}
	A convex function $g\in \cc{3}(\dom g)$, with $\dom g$ open, is said to be $(M_g,\nu)$-GSC of the order $\nu\in \rp$, with $M_g\in\rnn$, if $\forall x \in \dom{g}$, $\forall u,v \in \rr^p$, $\abs{\left<\thirdn{g}(x)[v]u,u\right>} \le M_g\norm{u}_x^2\norm{v}_x^{\nu-2}\norm{v}^{3-\nu}$, where $\nabla^3{g}(x)[v] \triangleq \lim\limits_{t\to 0} \left\{\left(\nabla^2{g}(x + tv)-\nabla^2{g}(v)\right)/{t}\right\}$.
\end{definition}
We now assume the following about $g$:
\begin{enumerate}[label=\enumlabel{G}, ref=\enumref{G}]
	\item The regularization function $g$ is convex and $(M_g,\nu)$-GSC.\label{ass:sc}
\end{enumerate}

The class of regularization functions satisfying condition \ref{ass:sc} includes the self-concordant smoothing functions for commonly used regularizers such as the $\ell_1$- and $\ell_2$-norms (see \defnref{def:smooth-function} below). The resulting smooth approximation has the key property that it \emph{epi-converges} to the original regularizer, providing useful features that can be exploited on the epigraph of $g$ for optimization.
\begin{definition}[\cite{adeoye2024self}]\label{def:smooth-function}
	The parameterized function $g\colon \rr^p \times \rp \to \rr$ is said to be a self-concordant smoothing function for a function $\bar{g}\in\pclsc{\rr^p}$ if $\epilim\limits_{\mu\downarrow 0} g = \bar{g}$ and $g(\cdot; \mu)$ is $(M_g,\nu)$-GSC, where $\mu\in\rp$ is a smoothing parameter.
\end{definition} 

For the regularized problem \eqref{eq:reg-prob} (with $g$ satisfying \ref{ass:sc}), the corresponding GGN update is obtained by augmenting the terms $Q_t$, $e_t$ and $J_t$, respectively by $0$, $1$ and $\nabla g(\theta_t)$ in the appropriate dimensions \citep{adeoye2023score}. Let us denote these augmented counterparts by $\Qaug_t$, $\eaug_t$ and $\Jaug_t$. We then write for the GGN
\begin{align}
	\theta_{t+1} = \theta_t - \alpha_t(\Jaug_t^\top \Qaug_t \Jaug_t + H_t)^{-1}\Jaug_t^\top \eaug_t,\label{eq:ggn-reg}
\end{align}
or in its convenient form for overparameterized models as \citep{adeoye2023score}
\begin{align}
	\theta_{t+1} = \theta_t -\alpha_t H_t^{-1}\Jaug_t^\top(I+\Qaug_t \Jaug_t H_t^{-1}\Jaug_t^\top)^{-1}\eaug_t,\label{eq:ggn-step-approx}
\end{align}
where $H_t \equiv \hessn g(\theta_t)$. Relative to the minimal assumptions required to control the dynamics of the network outputs for the unregularized case, \eg, positive definiteness of $G_t$ (which indeed holds in the overparameterized regime), we need the following standard regularity assumptions on $\hat{R}_s$, $\Qaug_t$ and $\eaug_t$ (see \appref{app:proof-main} for details on the regularity terms):
\begin{enumerate}[label=\enumlabel{R}, ref=\enumref{R}]
	\item $\hat{R}_s$ is $\gamma_R$-strongly convex, and has upper-bounded gradients and Hessian; $g$, $\Qaug_t$ and $\eaug_t$ are locally bounded.\label{ass:r}
\end{enumerate}
An important consequence of condition \ref{ass:a} is that, in addition to $g$ admitting a Lipschitz continuous gradient (see \lemref{thm:g-properties} in \appref{app:sc-properties}), we get that $J_t$ is (locally) Lipschitz continuous (see \appref{app:Lipschitz-J}). Then, together with \ref{ass:r} and the stability of $H_t$, we can control the key terms $H_t^{-1}\Jaug_t^\top(I+\Qaug_t \Jaug_t H_t^{-1}\Jaug_t^\top)^{-1}$ and $\eaug_t$ appearing in \eqref{eq:ggn-step-approx}.

Corresponding to \eqref{eq:kgd}, the overparameterized NN trained according to \eqref{eq:ggn-step-approx} evolves in discrete-time as
\begin{align}
	\Phi_{t+1} = \Phi_t - \alpha_t \Gaug_t \eaug_t, \label{eq:ggn-evolution}
\end{align}
where $\Gaug_t \triangleq J_tH_t^{-1}\Jaug_t^\top (I + \Qaug_t \Jaug_t H_t^{-1}\Jaug_t^\top)^{-1}\in \rr^{m\times(m+1)}$. One major observation about the behaviour of the dynamics of $G_t$ in \eqref{eq:kgd} in the overparameterized setting is its stability throughout the training process, which characterizes the optimizer's global optimality \citep{du2019gradient}. In the analysis of gradient descent, most stability and convergence results in the literature heavily rely on the (strictly positive) minimum eigenvalue of $G_t$. These kinds of results are not immediate with $\Gaug_t$ or, in general, with explicit regularization. However, the self-concordant condition on $g$ ensures that its Hessian is at least locally stable\footnote{As noted in the introduction, self-concordance helps to control the rate at which the Hessian of $g$ changes locally, and this property has been recently formalized and studied for the notion of local and global \emph{Hessian stability} in convex optimization (see, \eg, \cite{karimireddy2018global, gower2019rsn, carmon2020acceleration}).}, and hence for an appropriate parameterization of the NN, we can ensure the stability of the dynamics of $\Gaug_t$. In addition to these, also noteworthy is an immediate deduction from the Lipschitzness of $\varrho$ and $\nabla g$: the boundedness of the singular values of $\Jaug_t$ away from zero.

\section{Theoretical Result}
We study the convergence of self-concordant-regularized GGN for the one-hidden layer network. In line with the settings of Section~\ref{ss:ggn-settings}, the learning rate selection rule we consider throughout is
\begin{align}\label{eq:step-size}
	\alpha_t = \frac{\bar{\alpha}_t}{1+M_g\eta_t},
\end{align}
where $0 < \bar{\alpha}_t \le 1$ and $\eta_t = \|\nabla g(\theta_t)\|_{\theta_t}^*$. Without any emphasis on the particular choice of the target function $\Phi^*$, we assume it is given by any \emph{universally consistent}\footnote{Informally speaking, a learning algorithm is said to be universally consistent if the error of its estimate tends to zero as the sample size tends to infinity, for all distributions of the sample space such that the second moment of the output variable is finite. For a more precise context, see, \eg, \cite{caponnetto2007optimal} and the references therein.} algorithm as a minimum requirement. An example, in the case $\ell(\Phi,y) \triangleq \frac{1}{2}(\Phi - y)^2$ in \eqref{eq:reg-prob}, is the regularized least squares algorithm which, for a kernel prediction function $\Phi_\RF \equiv \Phi^*$ is defined by
\begin{align}\label{eq:rf}
	\rr^{n_0} \ni x \mapsto \Phi_\RF(x;(u^\infty,v^*)) \triangleq \sum_{i=1}^{n} v_i^*\varrho(u_i^\infty x).
\end{align}
This yields the estimator $v^*=[v_1^*, v_2^*,\ldots,v_n^*]\in\rr^n$, the unique minimizer of the $\ell_2$-regularized empirical loss $\frac{1}{2m}\sum_{i=1}^{m}(\Phi_\RF(x_i;v) - y_i)^2 + \frac{\lambda}{2}\norm{v}^2$, for some $\lambda\in\rnn$, where the entries of $u^\infty=[u_1^\infty, u_2^\infty,\ldots,u_n^\infty]^\top\in\rr^{n\times n_0}$ remain fixed iid random variables.

We state our main result of this section in \thmref{thm:main-result} below. The detailed proof is given in \appref{app:proof-main}. We let $\tPhi_t\in\rr^{m+1}$ denote the vector obtained by augmenting $\Phi_t$ by $1$. This $\tPhi_t$ corresponds to augmenting the rows of $J_t$ in the definition of $\Gaug_t$ by the vector whose entries are all zeros except the last entry which has the value $\tilde{\phi}_t = 1/\phi_{t-1}^{m+1}\in\rr$, where $\phi_t^{m+1}$ denotes the last entry of $H_t^{-1}\Jaug_t^\top (I + \Qaug_t \Jaug_t H_t^{-1}\Jaug_t^\top)^{-1}\eaug_t\in\rr^{n\times 1}$ at a time $t$. Let this augmented version of $J_t$ be denoted by $\tilde{J}_t$. Then, we define $\tG_t \triangleq \tilde{J}_tH_t^{-1}\Jaug_t^\top (I + \Qaug_t \Jaug_t H_t^{-1}\Jaug_t^\top)^{-1}\in \rr^{(m+1)\times(m+1)}$, and
\begin{align}
	\tPhi_{t+1} = \tPhi_t - \alpha_t \tG_t \eaug_t. \label{eq:ggn-evolution-aug}
\end{align}
We also let $\tPhi^*\in\rr^{m+1}$ denote the vector obtained by augmenting $\Phi_t^*$ by $0$. Additional regularity terms are explicitly defined in \appref{app:proof-main}.

\begin{theorem}\label{thm:main-result}
	Suppose that conditions \ref{ass:a}, \ref{ass:sc} and \ref{ass:r} hold for problem \eqref{eq:reg-prob}. Let $\Phi^*$ be a universally consistent target function that best interpolates the training data set $S$. Let $B_R$, $B_\Phi$, $B_g$, $d_g$, $d_q$, $\beta$, $D_g$ and $D_R$ be the regularity terms given by condition \ref{ass:r}, and denote $\hat{\beta}_m\triangleq \smin{J}\equiv \smin{J^\top}$. Then, in the ellipsoid $\calE_r(\theta_0)$ for some $r\in\rp$ and an initialization $\theta_0$, the regularized GGN for problem \eqref{eq:reg-prob} satisfies the following properties:
	\begin{enumerate}[label=\enumlabel{P}, ref=\enumref{P}]
		\item
		Fix $0<\bar{\alpha}_t\equiv\bar{\alpha} < 1$, and choose $\begin{aligned}
			T \triangleq \frac{1}{\bar{\alpha}}\log(\|\tPhi_0 - \tPhi^*\|^2/\epsilon)
		\end{aligned}$ for any $\epsilon \in (0,1)$. It holds that $\begin{aligned}
			\norm{\tPhi_T - \tPhi^*}^2 \le \epsilon + 1
		\end{aligned}$ after $T$ iterations, if $1+M_g\eta_t \le \|\tG_t\|_F$ and $|\tG_{22}|\ge |\langle\tG_{21}^\top + \tG_{12},\tilde{v}\rangle|$ for some $\tilde{v}$ depending on $t$ when $t\le T$ where, given a $2\times 2$ block partitioning of $\tG_t$, $\tG_{22}\in\rr^{1\times 1}$, $\tG_{21}\in\rr^{1\times (m+1)}$, and $\tG_{12}\in\rr^{(m+1)\times 1}$ respectively denote the lower right, lower left and upper right blocks of $\tG_t$,  \label{p:result1}
		\item
		$\begin{aligned}
			\calL(\theta_{t+1}) \le \calL(\theta_t) - \left[\vartheta L_{D_t}^2(1 + D_g\varpi_t) - \xi L_{D_t}\right]
		\end{aligned}$, for $t\in\rnn$, where $L_{D_t} \triangleq \frac{\alpha_t\beta\hat{\beta}_1D_g}{d_g(D_g + d_q\hat{\beta}_m^2)}$, $\xi \triangleq B_RB_\Phi + B_g$, $\vartheta \triangleq B_\Phi^2(\gamma_R - D_R)$, $\varpi_t \triangleq \omega_\nu(d_\nu(\theta_t,\theta_{t+1})) - \omega_\nu(-d_\nu(\theta_t,\theta_{t+1}))$, $\omega_\nu$ is an increasing univariate function, $d_\nu$ is a scaled metric term associated with the self-concordance of $g$, and we assume $d_\nu(\theta_t,\theta_{t+1})<1$. \label{p:result2}
	\end{enumerate}
\end{theorem}

\begin{remark}
	The condition that $1+M_g\eta_t \le \|\tG_t\|_F$ in \ref{p:result1} highlights an important aspect of the regularization. We often want to control the regularization strength via a parameter $\tau\in\rp$. Consider this general case in which $g$ takes the form $g(\theta) = \tau \bar{g}(\theta)$. Here, we only require that $\bar{g}$ is GSC so that $g$ satisfies \ref{ass:sc}. Observe that $\|\tG_t\|_F \le \sqrt{(m+1)\lambda_1(\tG_t^\top\tG_t)}$ for all $t$. Then since $M_g\eta_t$ can become arbitrarily small, it is only reasonable to choose $\tau$ satisfying $1+\tau M_g\eta_0 \le \sqrt{(m+1)\lambda_1(\tG_0^\top\tG_0)}$ in order to have the theoretical guarantee. In this setting, we essentially rely on the local stability of $\tG_t$ via overparameterization and the self-concordance of $g$.
\end{remark}

Since the direct relation between GGN and NTK probably breaks with an explicit regularization, our main proof step in \thmref{thm:main-result} involves analyzing a partitioning of the matrix $\Gaug_t$. In this way, we determine what conditions on the separate blocks help to combine certain spectral properties of $\tG_t$ with our regularity conditions and the self-concordance of $g$. The second result becomes almost immediate in the optimization scope under the regularity conditions.

Under the strong convexity assumption on $\hat{R}_s$, the global convergence of GGN can be guaranteed in the case of no regularization, for example, by training only the last layer of the NN given the property of \emph{no blow-up} of the GGN dynamics \citep[Proposition 1, Proposition 3]{arbel2023rethinking}. Consider the matrix $G^\infty$ whose $(i,j)$-th entry is given by $\langle \gradn_v{\Phi(x_i;(v,u^\infty)), \gradn_v{\Phi(x_j;(v,u^{\infty}))}}\rangle$, where $u^\infty$ is as defined in \eqref{eq:rf}. The main observation here is that the function $v \mapsto r(v) = \hat{R}_s(\Phi(\cdot;(v,u^\infty)))$ can be shown to satisfy a certain Polyak-{\L}ojasiewicz (PL) inequality, that is \citep[Proposition 4]{arbel2023rethinking} $\frac{1}{2}\norm{\gradn_v{l(v)}}^2 \ge \gamma_R\sigma_{\infty_n}^2(l(v) - \hat{R}_s(\Phi^*))$, where $\sigma_{\infty_n}$ is the minimum singular value of $G^\infty$. With a self-concordant regularization function $g$, this kind of global property is retained, provided that $g$ and $\hat{R}_s$ do not conflict. Consider, for example, the sublevel set $\calS_{g(\theta)}(g) \triangleq \set{\bar{\theta}\in \dom{g} \st g(\bar{\theta})\le g(\theta)}$ of $g$. Then, following \cite[Theorem 4]{sun2019generalized}, we get that $\calS_{g(\theta)}(g)$ is bounded for $\nu \in [2,3]$, and hence $g$ attains its minimum.

\section{Experiments}\label{ss:experiments}
We present numerical results from experiments performed on GGN with self-concordant regularization (GGN-SCORE) for overparameterized NNs on synthetic datasets as well as on the MNIST dataset. Results of additional experiments on the FashionMNIST and three UCI datasets are reported in \appref{app:add-mnist}. The code to reproduce the experimental results is available at \url{https://github.com/adeyemiadeoye/ggn-score-nn}.

\paragraph{Experimental setup.} We consider the \emph{teacher-student} setting in which $\Phi$ defined by \eqref{eq:nn} with $\kappa(n) = 1/\sqrt{n}$ is the student NN, while the teacher NN is the target function $\Phi^*$, a one-hidden layer NN given as
\begin{align}\label{eq:teacher-nn}
	\rr^{n_0} \ni x \mapsto \Phi^*(x;\theta^*) \triangleq \sum_{i=1}^{n^*} v_i^*\varrho(u_i^* x),
\end{align}
where $\theta^* \equiv (u^*,v^*)$. In both teacher and student networks, we use the SiLU activation function \citep{elfwing2018sigmoid} $\varrho(x) \triangleq x/(1+\exp(-x))$. In each experiment, we generate $m$ training data points $(x_i,y_i)_{i\in [m]}$, where the inputs $x_i$ are uniformly sampled on the unit sphere $\mathbb{S}^{n_0 - 1}\triangleq \{x \st \|x\|=1\}$ and the corresponding target outputs are given by $y_i = \Phi^*(x_i;\theta^*)$. The weights of the teacher NN are randomly generated as in \cite{chizat2019lazy}: they are normalized random weights satisfying $\|v^*_iu^*_i\| = 1$ for $i=1,\ldots,n^*$. The student NN is initialized with randomly generated weights from the Gaussian distribution. In all the experiments, we fix $n=500$ and $n^*=5$. The student NN is trained by minimizing the regularized empirical risk in \eqref{eq:reg-prob} with the squared loss $\ell$ (the empirical risk is unregularized for GD), and we consider regularization of the form $g(\theta) = \tau \bar{g}(\theta)$, where $\tau \in \rp$ and $\bar{g}$ is given by \citep[Example 1]{adeoye2024self}: $\bar{g}(\theta) = (\infconv{\|\cdot\|_1}{h_\mu})(\theta) = \sum_{i=1}^{n}\frac{\mu^{2}-\mu  \sqrt{\mu^{2}+\theta_i^{2}}+\theta_i^{2}}{\sqrt{\mu^{2}+\theta_i^{2}}}$, $h_\mu(\cdot) \triangleq \mu h(\cdot/\mu)$, $h(\theta) = \sum_{i=1}^{n} ((1+\abs{\theta_i}^2)^{1/2}-1)$, which gives the $(M_g,\nu)$-GSC function
\begin{align}
	g(\theta) = \tau \sum_{i=1}^{n}\frac{\mu^{2}-\mu  \sqrt{\mu^{2}+\theta_i^{2}}+\theta_i^{2}}{\sqrt{\mu^{2}+\theta_i^{2}}},\label{eq:g-example}
\end{align}
with $M_g=2\mu^{-0.7}p^{0.2}$, $\nu = 2.6$ (see \lemref{thm:g-properties}). We choose $\mu = 1/\kappa(n)$ and $\tau = 10^{-4}$, except for where we consider different values for comparison. We set $\bar{\alpha}_t \equiv \bar{\alpha} = 0.95$ in \eqref{eq:step-size} for GGN and use a learning rate of $1$ for GD. All experiments are performed on a laptop with $16 \times 2.30$GHz Intel Core i7-11800H CPU and 32GB RAM.

\subsection{Results and discussion}
\paragraph{Test loss vs. smoothing parameter.} We compare the performance of GGN-SCORE for different values of the regularization smoothing parameter $\mu$ evenly spaced in the range $[10^{-3},10]$, giving $41$ different values in total. We use a reasonable amount of training samples, $500$, which allows to perform several independent runs for each value of $\mu$ considered. We use a test size of $1000$ to measure generalization of the student NN for each $\mu$. We perform $10$ independent runs for each value of $\mu$ considered and took the average value of the results. These are shown in \figref{fig:testloss-vs-mu-tau}. We observe that larger values of $\mu$ yields better performance in the optimization scope and also better generalization. This result is quite intuitive, since by definition of the regularization function, the size of $\mu$ should scale with the size of the variable $\theta$ in order to have an adequate smooth approximation of the original nonsmooth function. For this reason, it is recommended to choose $\mu = c/\kappa(n)$ for any $c > 0$ when scaling with $\kappa(n)=1/\sqrt{n}$.
\begin{figure*}[h!]
	\centering
	\begin{subfigure}[b]{0.4\textwidth}
		\centering
		\includegraphics[width=\textwidth]{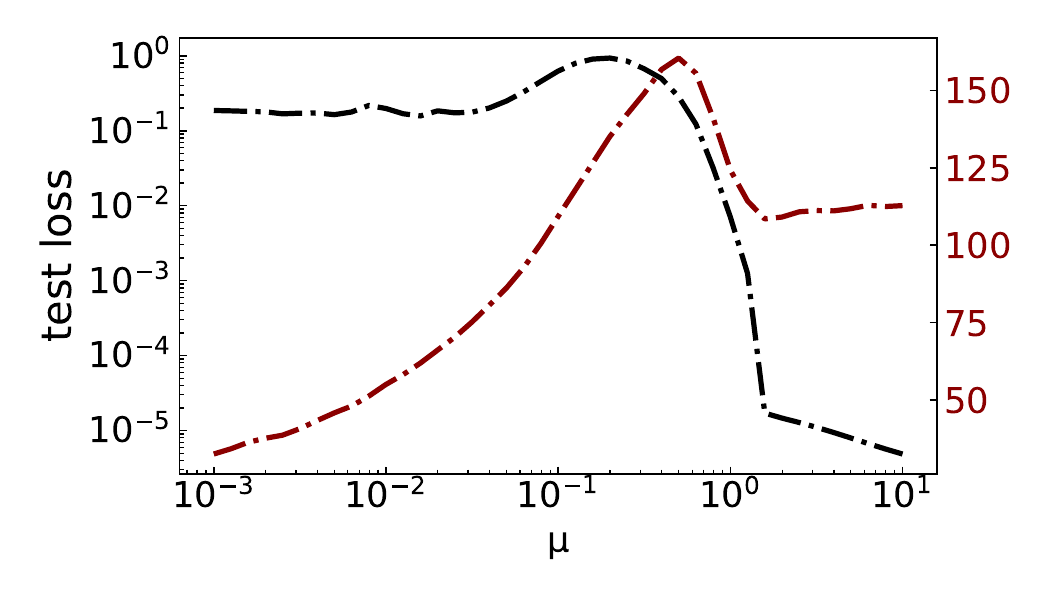}
	\end{subfigure}
	\begin{subfigure}[b]{0.4\textwidth}
		\centering
		\includegraphics[width=\textwidth]{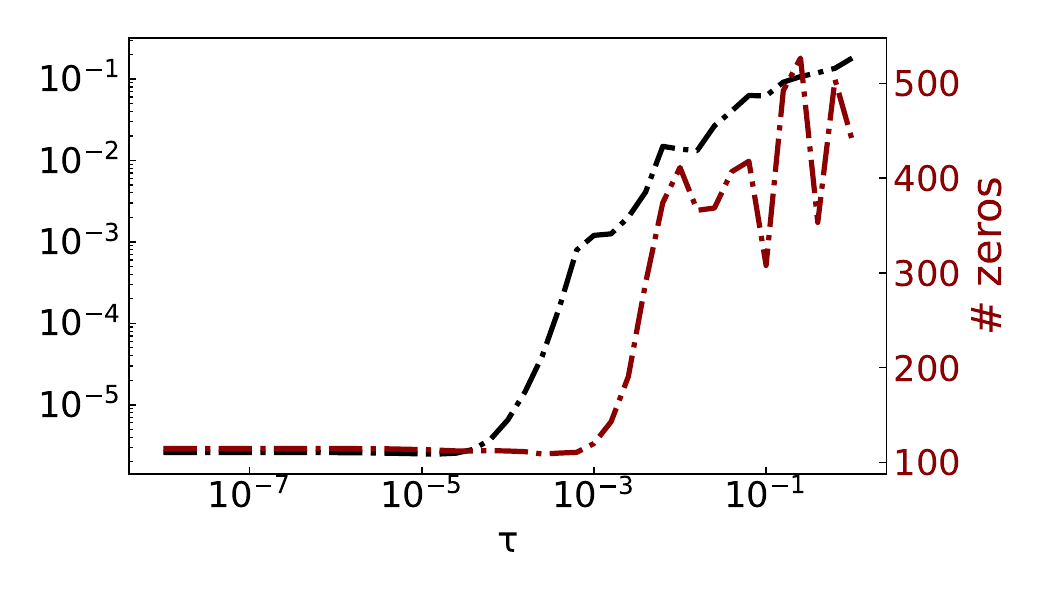}
	\end{subfigure}
	\caption{Performance of GGN-SCORE with $g(\theta)$ as in \eqref{eq:g-example}. \textbf{Left:} Results for different values of $\mu$, with $\tau=10^{-4}$. \textbf{Right:} Results for different values of $\tau$, with $\mu = 1/\sqrt{n}$. Results are averaged over $10$ independent runs for each value of $\mu$ and $\tau$, resp.; the total computation time is $\sim 21$ hours, $2$ minutes on CPU.}
	\label{fig:testloss-vs-mu-tau}
\end{figure*}
\paragraph{Test loss vs. regularization strength.} We study the influence of the regularization strength on the evolution of the test loss within the optimization loop of GGN-SCORE. We consider different values of $\tau$ evenly spaced in the range $[10^{-8},1]$, giving $41$ different values in total, and set $\mu = 1/\sqrt{n}$. Here, we also use a training size of $500$ and a test size $1000$ which allows to perform several independent runs for each value of $\tau$ considered. We perform $10$ independent runs for each value of $\tau$ and compute the average value of the results. These average values are shown in \figref{fig:testloss-vs-mu-tau}. We observe that smaller values of $\tau$ yield smaller training and test errors for the overparameterized NN. This observation corroborates with the analysis of \cite{wei2019regularization} for GD. However, contrarily to \cite{wei2019regularization}, what we observe for the GGN-SCORE is not an arbitrarily small regularization strength to achieve a good generalization performance. In fact, any value of $\tau$ slightly smaller than $10^{-4}$ in our experiment gives a similar generalization error as the choice $\tau = 10^{-6}$ (and smaller). \figref{fig:testloss-vs-mu-tau} also displays the average number of zero entries in the value of $\theta$ at the end of training. As observed, larger values of $\tau$ yields a sparser/simpler model. In principle, a desirable value of $\tau$ is one which helps to avoid overfitting of the NN model such that a simpler model implies better generalization.
\begin{figure*}[h!]
	\centering
	\begin{subfigure}[b]{0.3\textwidth}
		\centering
		\includegraphics[width=\textwidth]{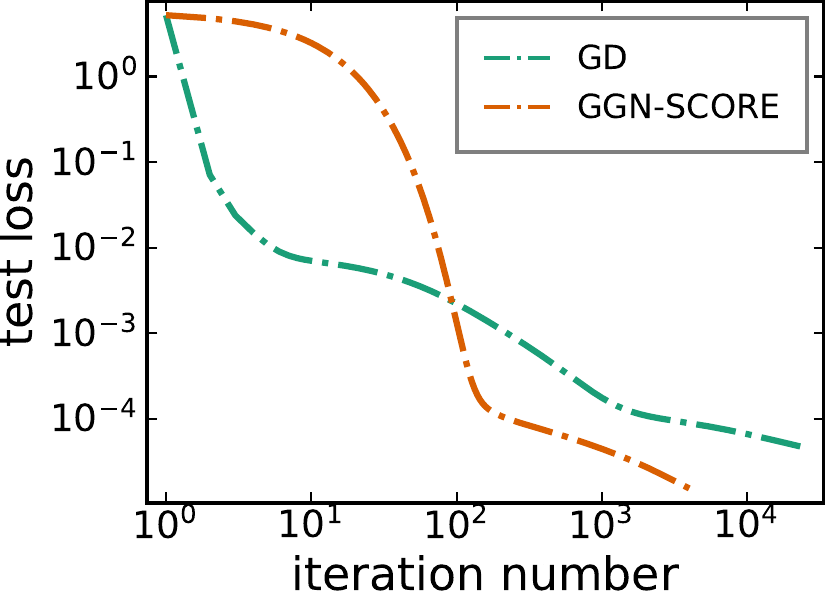}
	\end{subfigure}
	\hfil
	\begin{subfigure}[b]{0.3\textwidth}
		\centering
		\includegraphics[width=\textwidth]{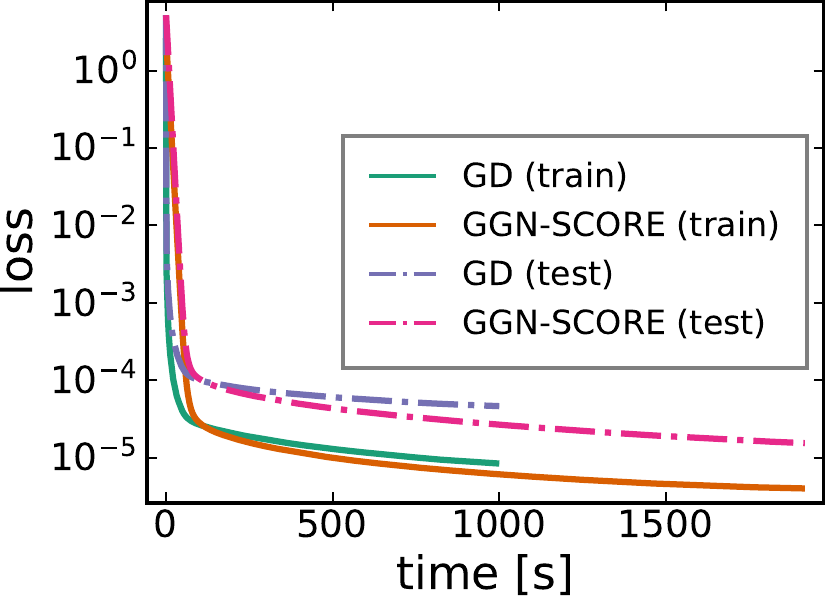}
	\end{subfigure}
	\caption{Performance of GD and GGN-SCORE per iteration number (\textbf{left}) and time in seconds (\textbf{right}) with $g(\theta)$ as in \eqref{eq:g-example} for GGN-SCORE, $\tau=10^{-4}$, $\mu = 1/\sqrt{n}$.}
	\label{fig:traintestloss-vs-epoch-time}
\end{figure*}
\paragraph{Performance comparison in the optimization loop.} We generate training and test datasets of sizes $1000$ and $2000$, respectively, and compare the training and test losses per iteration and time in seconds between GD and GGN-SCORE for training the student NN. The results are displayed in \figref{fig:traintestloss-vs-epoch-time}. The dimension of the input data in this experiment is $20$, and the number of hidden neurons for the student network is $500$. Choosing a much smaller number of hidden neurons $n^*=5$ for the teacher network keeps the optimization loop in the overparameterized regime. For the GD, we use a learning rate of $1$ which yields a much better performance than smaller values. While larger learning rates could yield faster learning at the beginning of training, we notice traces of divergence later on; a learning rate of $1$ gives a reasonably good descent and good performance of GD. We run GD for a total of $10000$ steps and GGN-SCORE for a total of $4000$ steps. GGN is well-known for its faster convergence in terms of number iterations, while sometimes, we may have to train for a longer time. Results here show that we do not trade total training time for better performance with our GGN-SCORE setup.
\subsection{Experiments on real datasets}\label{sec:mnist}
The computations involved in full-batch GD and/or GGN are intractable on real-world datasets. We study the performance of GGN-SCORE in the mini-batch setting on the standard MNIST dataset \cite{lecun2010mnist} with $n_0=784$, $m=60000:10000$ (training:test splits). Experimental results in the teacher-student setup according to \cite[Appendix C.4]{arbel2023rethinking} with the SiLU activation function are reported in \appref{app:add-mnist}. Additional experiments on three UCI datasets, as well as those on the FashionMNIST dataset, are also considered in the appendix. Here, we consider a NN of the form \eqref{eq:nn} with a hidden size $n=512$, a scaling $\kappa(n)=1/\sqrt{n}$ and the ReLU \cite{nair2010rectified} activation function. The NN is initialized with randomly generated weights from the Gaussian distribution, and is trained with the squared loss. The regularization function $g$ used in GGN-SCORE is given by \eqref{eq:g-example}. All results shown for GGN-SCORE are for a training batch size of $16$ (\ie, $3750$ training steps) and a single epoch. In addition to the test loss and prediction accuracy of the trained model, we adopt a time-invariance ``T-I" measure, representing the average proportion (in percentage) of the entries of the \emph{pre-activation} $a_l\in \rr^{n\times n_0}$ that satisfy $\sign(a_{ij}^{\text{start}}) = \sign(a_{ij}^{\text{final}})$, where $\sign$ is the \emph{signum function}, $a_{ij}^{\text{start}}$ are positional entries of $a_l$ at initialization and $a_{ij}^{\text{final}}$ are its entries at the end of training. This metric was used in \cite{chizat2019lazy} to measure the ``stability of activations" where high values indicate an effective linearization of the NN model. See additional details and remark in \appref{app:ti-measure}.
\begin{figure}[h!]
	\centering
	\includegraphics[width=0.8\textwidth]{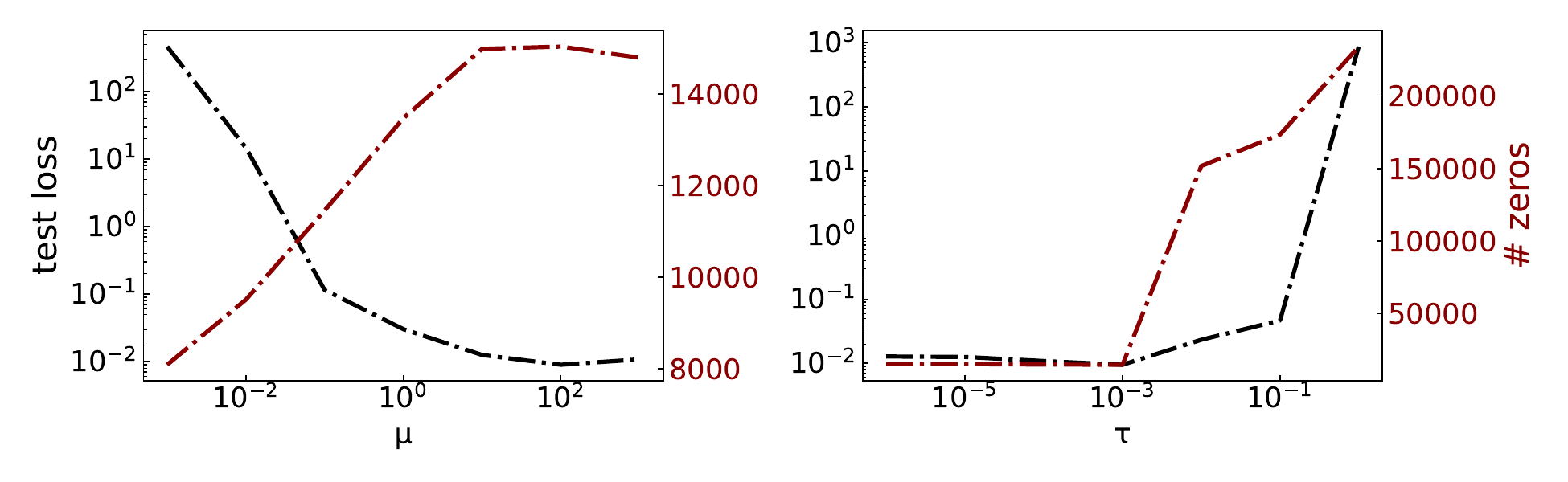}
	\caption{Test loss evaluation of the GGN-SCORE-trained NN on MNIST dataset for different values of the regularization smoothing parameter $\mu$ fixing $\tau=10^{-4}$ (\textbf{left}) and different values of the regularization strength $\tau$ fixing $\mu=1/\sqrt{n}$ (\textbf{right}). The regularization function $g(\theta)$ is given by \eqref{eq:g-example}.}
	\label{fig:testloss-vs-tau-mu-mnist}
\end{figure}
\paragraph{Influence of the regularization parameters.} We investigate the performance of the GGN-SCORE-trained model for different values of the regularization smoothing parameter $\mu$ and the regularization strength $\tau$ on MNIST dataset. First, we fix $\mu = 1/\sqrt{n}$ and measure the performance of the trained model for different values of $\tau$. Similarly, we fix $\tau = 10^{-4}$ and measure the trained model's performance for varying values of $\mu$. In each of the two cases, we use $7$ different values of $\tau$ and $\mu$ as is respectively shown in \figref{fig:testloss-vs-tau-mu-mnist} and \figref{fig:accuracy-mnist}. As in the case with synthetic datasets, optimal choices for $\tau$ and $\mu$ are seen to necessarily yield good generalization of the model, and are such that give a relatively simple model and stable dynamics (as indicated by the number of zeros in the parameters of the final optimized model and the T-I measure). The total computation time to generate the results in \figref{fig:testloss-vs-tau-mu-mnist} and \figref{fig:accuracy-mnist} is $\sim23$ hours, $5$ minutes on CPU.
\begin{figure*}[h!]
	\centering
	\begin{subfigure}[b]{0.3\textwidth}
		\centering
		\includegraphics[width=\textwidth]{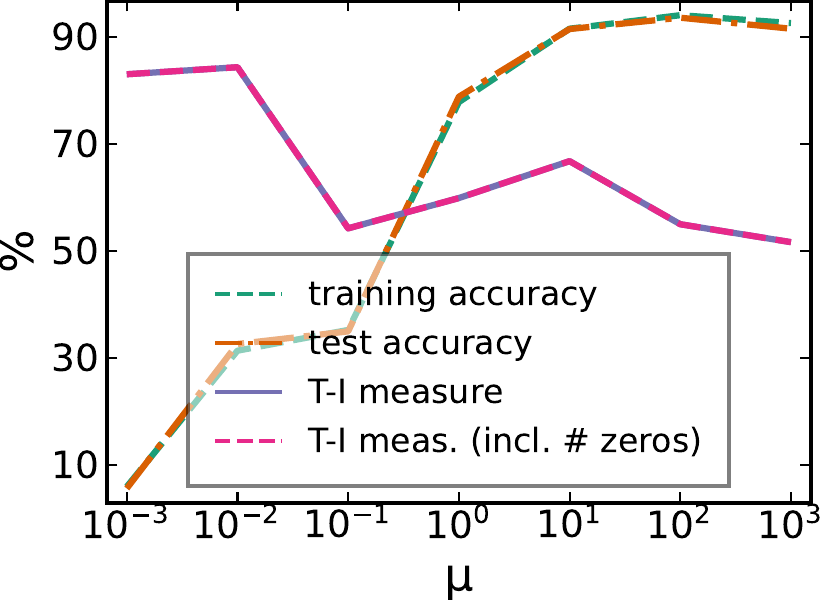}
	\end{subfigure}
	\hfil
	\begin{subfigure}[b]{0.3\textwidth}
		\centering
		\includegraphics[width=\textwidth]{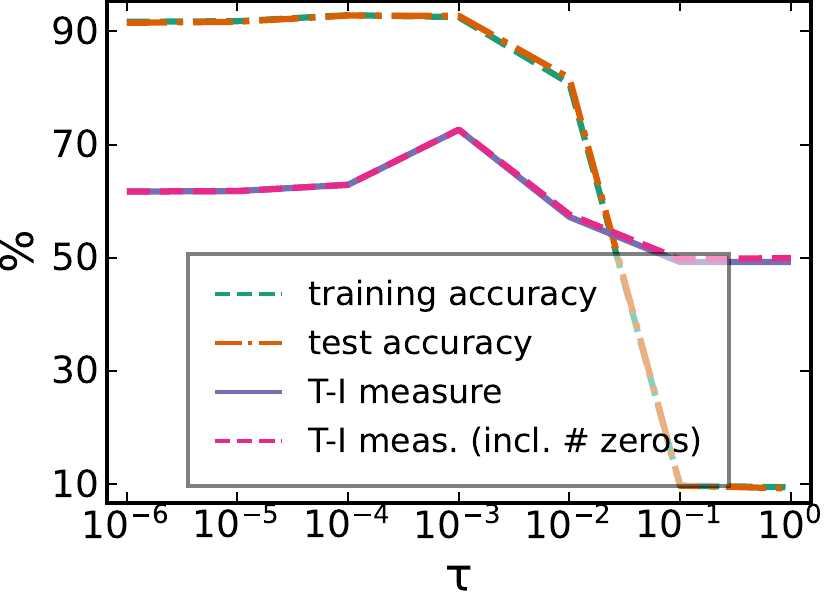}
	\end{subfigure}
	\caption{Evaluation of GGN-SCORE on MNIST dataset for different values of the regularization smoothing parameter $\mu$, fixing $\tau=10^{-4}$ (\textbf{left}), and different values of the regularization strength $\tau$, fixing $\mu=1/\sqrt{n}$ (\textbf{right}). The regularization function $g(\theta)$ is given by \eqref{eq:g-example}.}
	\label{fig:accuracy-mnist}
\end{figure*}
\section{Conclusion}
We studied a Generalized Gauss-Newton method for optimizing a two-layer neural network with \emph{explicit} regularization and for a specific type of two-layer neural network. We considered the class of generalized self-concordant regularization under which we proved convergence of the neural network predictions to the outputs of a given target function, and have quantified the decay of the problem's objective function throughout the training iterations. Our experimental findings revealed that good generalization of the optimized neural network model can be achieved with the regularization framework. In future research, we will further investigate the regularization framework with GPU-supported experiments for wider and deeper neural networks and higher dimensional datasets.


\bibliographystyle{unsrtnat}
\bibliography{references.bib}

\appendix

\section{Preliminary results}
\subsection{Lipschitz continuity of $J$}\label{app:Lipschitz-J}
Here, we look at the Lipschitz property of the Jacobian matrix $J$. For this, we need the following additional standard assumptions:
\begin{enumerate}[label=\enumlabel{J}, ref=\enumref{J}]
	\item $\|x\| = \|x^\top\| \le 1$.\label{ass:j1}
	\item $\exists L_\varrho$ such that $\|\varrho(\bar{x}) - \varrho(\tilde{x})\|\le L_\varrho\|\bar{x} - \tilde{x}\|$ and  $\|\varrho'(\bar{x}) - \varrho'(\tilde{x})\|\le L_\varrho\|\bar{x} - \tilde{x}\|$ for all $\bar{x},\tilde{x}\in\rr$.\label{ass:j2}
	\item $\exists L_v$ such that $\|v\| \le L_v$ at all time $t$ until the training is stopped.\label{ass:j3}
\end{enumerate}
Condition \ref{ass:j2} simply restates the Lipschitzness and smoothness assumptions in \ref{ass:a} explicitly. This kind of condition has been used, for example in \cite[Condition 3.1]{du2019gradient}, to show the stability of the training process of NNs via gradient descent. A consequence of the condition is that it also provides an upper bound on the gradients of $\varrho$, that is, $\|\varrho'(\bar{x})\|\le L_\varrho$ for all $\bar{x} \in \rr$ (see, \eg, \lemref{thm:grad-g-bound} below).
\begin{proposition}[Lipschitz constant of $J$; training both layers]
	Under assumptions \ref{ass:j1}, \ref{ass:j2} and \ref{ass:j3}, $J$ is $L_J$-Lipschitz continuous, where $L_J\triangleq m\kappa(n)(1+L_v) L_\varrho\sqrt{2}$.
\end{proposition}
\begin{proof}
	Let $(\bar{u}, \bar{v}) \equiv \bar{\theta}, (\tilde{u}, \tilde{v}) \equiv \tilde{\theta}$ for any $\bar{\theta}, \tilde{\theta} \in \rr^p$. We have
	\begin{align}
		\norm{J(\bar{\theta}) - J(\tilde{\theta})} &\le \kappa(n) \left(\sum_{i=1}^n \abs{\varrho'(\bar{u}_i x)\bar{v}_ix^\top - \varrho'(\tilde{u}_i x)\tilde{v}_ix^\top} + \sum_{i=1}^n \abs{\varrho(\bar{u}_i x) - \varrho(\tilde{u}_i x)}\right)\nonumber\\
		&\le \kappa(n) \left(\sum_{i=1}^n \left(\abs{\varrho'(\bar{u}_i x)\bar{v}_ix^\top - \varrho'(\tilde{u}_i x)\bar{v}_ix^\top} + \abs{\varrho'(\tilde{u}_i x)\bar{v}_ix^\top - \varrho'(\tilde{u}_i x)\tilde{v}_ix^\top}\right)\right.\nonumber\\ &\left.\qquad+ \sum_{i=1}^n \abs{\varrho(\bar{u}_i x) - \varrho(\tilde{u}_i x)}\right)\nonumber\\
		&\le \kappa(n) \left(\sum_{i=1}^n \left(\abs{\varrho'(\bar{u}_i x) - \varrho'(\tilde{u}_i x)}\abs{\bar{v}_i}\norm{x} + \abs{\bar{v}_i - \tilde{v}_i}\abs{\varrho'(\tilde{u}_i x)}\norm{x}\right) + \sum_{i=1}^n \abs{\varrho(\bar{u}_i x) - \varrho(\tilde{u}_i x)}\right)\nonumber\\
		&\le m\kappa(n)\left((1+L_v)L_\varrho\norm{\bar{u} - \tilde{u}} + L_\varrho\norm{\bar{v} - \tilde{v}}\right)\nonumber\\
		&\le m\kappa(n)(1+L_v)L_\varrho\left(\norm{\bar{u} - \tilde{u}} + \norm{\bar{v} - \tilde{v}}\right)\nonumber\\
		&= m\kappa(n)(1+L_v)L_\varrho\left(\norm{(\bar{u}, 0) - (\tilde{u},0)} + \norm{(0,\bar{v}) - (0,\tilde{v})}\right).\label{eq:exprJ1}
	\end{align}
	Next, we recall Peter-Paul inequality for two quantities $a,b \in \rnn$ which reads $2ab \le a^2 + b^2$, from which we obtain
	\begin{align}
		(a+b)^2 = a^2 + 2ab + b^2 \le a^2 + a^2 + b^2 + b^2  = 2(a^2 + b^2).\label{eq:exprJ2}
	\end{align}
	We also derive the expression
	\begin{align}
		\norm{(\bar{u}, 0) - (\tilde{u},0)}^2 + \norm{(0,\bar{v}) - (0,\tilde{v})}^2 &= \norm{(\bar{u}, 0)}^2 - 2\langle (\bar{u},0), (\tilde{u},0)\rangle + \norm{(\tilde{u},0)}^2\nonumber\\ &\quad+ \norm{(0, \bar{v})}^2 - 2\langle (0,\bar{v}), (0,\tilde{v})\rangle + \norm{(0,\tilde{v})}^2\nonumber\\
		&= \norm{(\bar{u}, \bar{v})}^2 - 2\langle (\bar{u},\bar{v}), (\tilde{u},\tilde{v})\rangle + \norm{(\tilde{u}, \tilde{v})}^2\nonumber\\
		&= \norm{(\bar{u}, \bar{v}) - (\tilde{u}, \tilde{v})}^2.\label{eq:exprJ3}
	\end{align}
	Now, using \eqref{eq:exprJ2} and \eqref{eq:exprJ3} in \eqref{eq:exprJ1} with $a=\norm{(\bar{u}, 0) - (\tilde{u},0)}$ and $b=\norm{(0,\bar{v}) - (0,\tilde{v})}$, we get
	\begin{align*}
		\norm{J(\bar{\theta}) - J(\tilde{\theta})} \le m\kappa(n)(1+L_v)L_\varrho\sqrt{2}\norm{(\bar{u}, \bar{v}) - (\tilde{u}, \tilde{v})},
	\end{align*}
	which proves the result.
\end{proof}
\subsection{Useful results on the generalized self-concordance of $g$}\label{app:sc-properties}
We define the following metric term for the regularization function $g$ (under condition \ref{ass:sc}). As is customary, our results are restricted to the case $\nu\in[2,3]$.
\begin{align}\label{eq:d-metric}
	d_\nu(\bar{\theta},\tilde{\theta}) \triangleq \begin{cases}
		M_g\norm{\tilde{\theta}-\bar{\theta}} &\text{if } \nu = 2,\\
		\left(\frac{\nu}{2}-1\right)M_g\norm{\tilde{\theta}-\bar{\theta}}_2^{3-\nu}\norm{\tilde{\theta}-\bar{\theta}}_{\bar{\theta}}^{\nu-2} &\text{if }\nu>2.
	\end{cases}
\end{align}
\begin{lemma}\cite[Proposition 10]{sun2019generalized}\label{thm:g-bound0}
	Under condition \ref{ass:sc}, we have for any $\bar{\theta},\tilde{\theta} \in \dom g$
	\begin{align}
		\omega_\nu(-d_\nu(\bar{\theta},\tilde{\theta}))\norm{\tilde{\theta}-\bar{\theta}}_{\bar{\theta}}^2 \le g(\tilde{\theta}) - g(\bar{\theta}) - \langle\gradn g(\bar{\theta}), \tilde{\theta}-\bar{\theta}\rangle \le \omega_\nu(d_\nu(\bar{\theta},\tilde{\theta}))\norm{\tilde{\theta}-\bar{\theta}}_{\bar{\theta}}^2,
	\end{align}
	in which, if $\nu>2$, the right-hand side inequality holds if $d_\nu(\bar{\theta},\tilde{\theta}) < 1$, and
	\begin{align}\label{eq:omega-nu}
		\omega_\nu(r) \triangleq \begin{cases}
			\frac{\exp(r)-r-1}{r^2} &\text{if } \nu = 2,\\
			\frac{-r - \ln(1-r)}{r^2} &\text{if } \nu=3,\\
			\frac{(1-r)\ln(1-r)+r}{r^2} &\text{if } \nu=4,\\
			\left(\frac{\nu-2}{4-\nu}\right)\frac{1}{r}\left[\frac{\nu-2}{2(3-\nu)r}\left((1-r)\frac{2(3-\nu)}{2-\nu}-1\right)-1\right] &\text{otherwise}.
		\end{cases}
	\end{align}
\end{lemma}

In case $g$ is nonsmooth and hence does not satisfy the condition \ref{ass:sc} but is closed, proper and convex, the following result from \cite[Proposition 2]{adeoye2024self} shows some properties of a self-concordant smoothing function for $g$ constructed in the sense of \defnref{def:smooth-function}.
\begin{lemma}\label{thm:g-properties}
	Let $\bar{g},h$ be two functions in $\pclsc{\rr^p}$. Suppose that $h$ is $(M_h,\nu)$-GSC and supercoercive, and define $g \triangleq \infconv{\bar{g}}{h_\mu}$ for all $\mu \in \rp$, where $h_\mu(\cdot) \triangleq \mu h\left(\frac{\cdot}{\mu}\right)$ and $\infconv{\bar{g}}{h_\mu}$ denotes the infimal convolution of $\bar{g}$ and $h_\mu$ defined by
	\begin{align}
		(\infconv{\bar{g}}{h_\mu})(\bar{\theta}) \triangleq \inf\limits_{\tilde{\theta}\in\rr^p}\left\{\bar{g}(\tilde{\theta})+h_\mu(\bar{\theta}-\tilde{\theta})\right\}. \label{eq:infconv}
	\end{align}
	Then,
	\begin{enumerate}[(i)]
		\item $g\in \pclsc{\rr^p}$ and is exact. \label{thm:exact}
		\item $g$ is $(M_g,\nu)$-GSC with
		\begin{align*}
			M_g =
			\begin{cases}
				p^{\frac{3-\nu}{2}}\mu^{\frac{\nu}{2}-2}M_h, \qquad \text{if } \nu \in (0,3],\\
				\mu^{4-\frac{3\nu}{2}}M_h, \qquad \text{if } \nu >3.
			\end{cases}
		\end{align*}
		\item $g$ is locally Lipschitz continuous.\label{thm:g-lip}
	\end{enumerate}
\end{lemma}
\begin{lemma}\label{thm:grad-g-bound}
	Let $g$ be a convex and (locally) $L$-Lipschitz function. Then,
	\begin{align*}
		\norm{\gradn{g}(\bar{\theta})} \le L,
	\end{align*}
	for some $\bar{\theta}$ in a set $\calX \subset \rr^p$.
\end{lemma}
\begin{proof}
	Take some $\tilde{\theta} = \bar{\theta} + \alpha\gradn{g(\bar{\theta})}$ for $\alpha\in\rp$ small enough. By the convexity and Lipschitzness of $g$, we have
	\begin{align*}
		\norm{\alpha\gradn{g}(\bar{\theta})}^2 &= \alpha^2\abs{\langle\gradn{g}(\bar{\theta}),\gradn{g}(\bar{\theta})\rangle}\\
		&= \alpha\abs{\langle \tilde{\theta}-\bar{\theta}, \gradn{g}(\bar{\theta})\rangle}\\
		&\le \alpha\abs{g(\tilde{\theta}) - g(\bar{\theta})}\\
		&\le \alpha L\abs{\tilde{\theta} - \bar{\theta}}\\
		&= \alpha^2L\norm{\gradn{g(\bar{\theta})}},
	\end{align*}
	which completes the proof.
\end{proof}
\section{Proof of the main result}\label{app:proof-main}
\paragraph{Detailed regularity assumptions in \ref{ass:r}.} We detail missing regularity terms in condition \ref{ass:r} as follows. For this, we define the ball $\calB_{r_0}(\theta_0)\subset\calE_{r}(\theta_0)$ for some initialization $\theta_0$. Note that, corresponding to the ellipsoid $\calE_{r}(\theta_0)$, we compute local norms with respect to $g$.
\begin{enumerate}[label=\enumlabel{RR}, ref=\enumref{RR}]
	\item $\hat{R}_s(\Phi(\cdot,\bar{\theta})) \ge \hat{R}_s(\Phi(\cdot,\tilde{\theta})) + \langle \gradn_\Phi{\hat{R}_s(\Phi(\cdot,\tilde{\theta}))}, \Phi(\cdot,\bar{\theta}) - \Phi(\cdot,\tilde{\theta})\rangle + \frac{\gamma_R}{2}\norm{\Phi(\cdot, \bar{\theta}) - \Phi(\cdot, \tilde{\theta})}^2$ $\forall \bar{\theta}, \tilde{\theta} \in \rr^p$, and $\exists B_R, D_R$ such that $\norm{\gradn_\Phi{\hat{R}_s}(\Phi(\cdot,\bar{\theta}))}\le B_R$, $\norm{\hessn_\Phi{\hat{R}_s}(\Phi(\cdot,\bar{\theta}))}_{op}\le D_R$ $\forall \bar{\theta} \in \calB_{r_0}(\theta_0)$.\label{ass:r1}
	\item $d_gI \le H_t \le D_gI$ with $D_g \ge d_g >0$, $d_qI \le \Qaug_t \le D_qI$ with $D_q\ge d_q \ge 0$, and $\norm{\eaug_t} \le \beta$ $\forall \theta_t \in \calB_{r_0}(\theta_0)$.\label{ass:r2}
\end{enumerate}
Using the Lipschitness of $\varrho$, one can easily find some $B_\Phi$ satisfying $\norm{\Phi(\cdot, \bar{\theta}) - \Phi(\cdot, \tilde{\theta})}\le B_\Phi\norm{\bar{\theta} - \tilde{\theta}}$ for some $\bar{\theta}, \tilde{\theta}\in \rr^p$ at least near the initialization. Hence, we do not impose this regularity property as an additional assumption. Subsequently, we recall the following notations: $L_{D_t} \triangleq \frac{\alpha_t\beta\hat{\beta}_1D_g}{d_g(D_g + d_q\hat{\beta}_m^2)}$, $\xi \triangleq B_RB_\Phi + B_g$, $\vartheta \triangleq B_\Phi^2(\gamma_R - D_R)$, $\varpi_t \triangleq \omega_\nu(d_\nu(\theta_t,\theta_{t+1})) - \omega_\nu(-d_\nu(\theta_t,\theta_{t+1}))$, $\hat{\beta}_m\triangleq \smin{\Jaug_t}$. We also introduce the notations $\hat{\beta}_1\triangleq \smax{\Jaug_t}$ and $\delta_t \triangleq \theta_{t+1} - \theta_t$.

\begin{lemma}\label{thm:step-dist}
	Under assumption \ref{ass:r2}, we have
	\begin{align*}
		\norm{\delta_t} \le L_{D_t}, \qquad \norm{\delta_t}_{\theta_t} \le \sqrt{D_g}L_{D_t}.
	\end{align*}
\end{lemma}
\begin{proof}
	We obtain the following estimate
	\begin{align}
		\norm{H_t^{-1}\Jaug_t^\top} &\le \norm{H_t^{-1}}\norm{\Jaug_t^\top} \le \frac{\beta_1}{d_g}\label{eq:HJ_est}.
	\end{align}
	We have also
	\begin{align}
		\norm{I + \Qaug_t\Jaug_t H_t^{-1}\Jaug_t^\top} \ge 1 + \frac{d_q\hat{\beta}_m^2}{D_g}.\label{eq:IQJHJ_est}
	\end{align}
	From \eqref{eq:ggn-step-approx}, we have
	\begin{align}
		\norm{\delta_t} \le \alpha_t \norm{H_t^{-1}\Jaug_t^\top(I+\Qaug_t \Jaug_t H_t^{-1}\Jaug_t^\top)^{-1}\eaug_t}.\label{eq:stepnorm}
	\end{align}
	Using \eqref{eq:HJ_est} and \eqref{eq:IQJHJ_est} in \eqref{eq:stepnorm}, we obtain
	\begin{align*}
		\norm{\delta_t} \le L_{D_t}.
	\end{align*}
	The result follows, noting that $\norm{\delta_t}_{\theta_t} = \norm{H_t^{1/2}\delta_t}$ by definition.
\end{proof}
We obtain the following slightly loose estimate of the Lipschitz constant of the objective function $\calL$ in problem \eqref{eq:reg-prob}.
\begin{lemma}\label{thm:L-Lipschitz}
	Let $g$ be constructed under the settings of \lemref{thm:g-properties} such that condition \ref{ass:sc} holds. Then, under the additional condition \ref{ass:r1}, the objective function $\calL$ is $(B_RB_\Phi + B_g)$-Lipschitz continuous, where $B_g$ is the local Lipschitz constant of $g$ in \lemref{thm:g-properties}\ref{thm:g-lip}.
\end{lemma}
\begin{proof}
	By the convexity of $\hat{R}_s$ and $g$, we have
	\begin{align*}
		\hat{R}_s(\Phi(\cdot,\tilde{\theta})) \ge \hat{R}_s(\Phi(\cdot,\bar{\theta})) + \langle\gradn{\hat{R}_s(\Phi(\cdot,\tilde{\theta}))}, \Phi(\cdot,\tilde{\theta}) - \Phi(\cdot,\bar{\theta})\rangle,\quad g(\tilde{\theta}) \ge g(\bar{\theta}) + \langle\gradn{g(\tilde{\theta})}, \tilde{\theta} - \bar{\theta}\rangle,
	\end{align*}
	for some $\tilde{\theta},\bar{\theta}$ in the vicinity of $\theta_0$. Then, using $\calL \triangleq \hat{R}_s + g$ and the Cauchy-Schwarz inequality, we get
	\begin{align*}
		\abs{\calL(\bar{\theta}) - \calL(\tilde{\theta})}
		& \le \norm{\gradn_\Phi{\hat{R}_s(\Phi(\cdot,\bar{\theta}))}}\norm{\Phi(\cdot,\bar{\theta}) - \Phi(\cdot,\tilde{\theta})} + \norm{\gradn{g(\bar{\theta})}}\norm{\bar{\theta} - \tilde{\theta}}.
	\end{align*}
	By assumption \ref{ass:r1}, we have $\norm{\gradn_\Phi{\hat{R}_s(\Phi(\cdot,\bar{\theta}))}} \le B_R$. By \lemref{thm:g-properties}\ref{thm:g-lip}, $g$ is locally Lipschitz, and hence we have $\norm{\gradn{g(\bar{\theta})}} \le B_g$ for some $B_g$, according to \lemref{thm:grad-g-bound}. Then, using the local Lipschitz property of $\Phi$, we obtain
	\begin{align*}
		\abs{\calL(\bar{\theta}) - \calL(\tilde{\theta})} &\le B_RB_\Phi\norm{\bar{\theta} - \tilde{\theta}} + B_g\norm{\bar{\theta} - \tilde{\theta}}\\
		&= (B_RB_\Phi + B_g)\norm{\bar{\theta} - \tilde{\theta}}.
	\end{align*}
\end{proof}
The following result from \cite[Lemma 1]{wang1986trace} provides a useful inequality for the trace of the product of two symmetric matrices, one of which is positive semidefinite.
\begin{lemma}\label{thm:mat-prod}
	Let $P,Q\in\rr^{n\times n}$. If $P=P^\top\succeq 0$ and $Q$ is symmetric, then
	\begin{align*}
		\tr(P)\lambda_n(Q) \le \tr(PQ) \le \tr(P)\lambda_1(Q).
	\end{align*}
\end{lemma}
The next result concerns the $2\times 2$ block partitioning of $\tG_t$, and characterizes the positive-definiteness of its leading principal blocks. For this, we require that the function $g$ is such that $H_t\succ 0$. This, indeed, is a property of many functions constructed from the $\ell_1$-norm in the sense of \defnref{def:smooth-function}. An example is the pseudo-Huber function or the function $\bar{g}$ considered in Section~\ref{ss:experiments}.
\begin{lemma}\label{thm:G-positive}
	Consider a $2\times 2$ block partitioning of $\tG_t$, and let $\tG_{11}\in\rr^{m\times m}$, $\tG_{22}\in\rr^{1\times 1}$ respectively denote the upper left and lower right blocks. If $H_t\succ 0$, then it holds that $\tG_{22} \in \rp$ and $\tG_{11} \succ 0$.
\end{lemma}
\begin{proof}
	By the definition of $\tG_t$ and using \eqref{eq:ggn-reg}, we have $\tG_t = \tilde{J}_t(\Jaug_t^\top \Qaug_t \Jaug_t + H_t)^{-1}\Jaug_t^\top$. We note that for the squared loss that we consider, $Q_t$ is the identity matrix and that we can write $\tG_t = \tilde{J}_t(J_t^\top Q_t J_t + H_t)^{-1}\Jaug_t^\top=\tilde{J}_t(J_t^\top J_t + H_t)^{-1}\Jaug_t^\top$. Notice the removal of the augmentations, as the last diagonal entry of $\Qaug_t$ is zero. We have $\langle\hat{v},J_t^\top J_t\hat{v}\rangle = \|J\hat{v}\|^2 \ge 0$ for all non-zero $\hat{v}\in\rr^n$, and hence $B_t \triangleq J_t^\top J_t + H_t \succ 0$. Next, observe that $\tG_{11}$ results from removing the augmentations on $\tilde{J}$ and $\Qaug_t$ in $\tG_t$, that is, $\tG_{11} \equiv J_tB_t^{-1}J_t^\top$. Let $\hat{u} \triangleq B_t\hat{v}$; we have $\langle \hat{u},B_t^{-1}\hat{u}\rangle = \langle B_t\hat{v},B_t^{-1}B_t\hat{v}\rangle = \langle \hat{v},B_t^\top \hat{v}\rangle \succ 0$. Then, in a similar way, if $J_t$ does not have all its entries equal to zero, we get that $\tG_{11} \succ 0$.
	
	To show $\tG_{22} \in \rp$, we note that since $B_t\succ 0$, it has a non-zero determinant, and hence by Sylvester's criterion, we have $\frac{1}{\det(B_t)}((-1)^{2n}M_{n,n}) > 0$, where $\det(B_t)$ denotes the determinant of $B_t$ and $M_{n,n}$ denotes the $(n,n)$-th minor of $B_t$. Then, $\tG_{22} \in \rp$ follows from the definition of $\tilde{J}$.
\end{proof}
We are now ready to prove our main result.
\paragraph{Proof of \thmref{thm:main-result}.}
\begin{proof}
	Consider the time evolution of the regularized NN given by \eqref{eq:ggn-evolution}. Using the augmentation specified by \eqref{eq:ggn-evolution-aug}, we have
	\begin{align}
		\norm{\tPhi_{t+1} - \tPhi^*}^2 &= \norm{\tPhi_t - \alpha_t \tG_t \eaug_t - \tPhi^*}^2\nonumber\\
		&= \norm{\tPhi_t - \tPhi^*}^2 - 2\alpha_t\langle \tPhi_t - \tPhi^*,\tG_t(\tPhi_t - \tPhi^*)\rangle + \alpha_t^2\norm{\tG_t(\tPhi_t - \tPhi^*)}^2. \label{eq:norm-est0}
	\end{align}
	Let us partition $\tPhi_t - \tPhi^*$ and $\tG_t$ as follows (omitting dependence on $t$ in the blocks for brevity):
	\begin{align}
		\tPhi_t - \tPhi^* \equiv \left[
		\begin{array}{c}
			\tPhi_1 \\ \hdashline[2pt/2pt]
			\tPhi_2
		\end{array}
		\right], \quad \tG_t \equiv \left[
		\begin{array}{c;{2pt/2pt}c}
			\tG_{11} & \tG_{12} \\ \hdashline[2pt/2pt]
			\tG_{21} & \tG_{22}
		\end{array}
		\right],\label{eq:partition}
	\end{align}
	where $\tPhi_1 = \Phi_t - \Phi^* \in \rr^m, \tPhi_2 = 1$ and hence $\tG_{11} \in \rr^{m\times m}$. Then, we have
	\begin{align}
		\langle \tPhi_t - \tPhi^*,\tG_t(\tPhi_t - \tPhi^*)\rangle &= \langle \tPhi_1, \tG_{11}\tPhi_1\rangle + \langle \tPhi_2, \tG_{21}\tPhi_1\rangle + \langle \tPhi_1, \tG_{12}\tPhi_2\rangle + \langle \tPhi_2, \tG_{22}\tPhi_2\rangle\nonumber\\
		&= \langle \tPhi_1, \tG_{11}\tPhi_1\rangle + \langle\tG_{21} + \tG_{12}^\top, \tPhi_1\rangle + \tG_{22}, \label{eq:norm-est1}
	\end{align}
	where we have used $\tPhi_2 = 1$. Recall that by \lemref{thm:G-positive}, we get $\tG_{22} \in \rp$ and $\tG_{11} \succ 0$.
	
	Using the block partitioning of $\tG_t$ in \eqref{eq:partition}, the product $\tG_t^\top\tG_t$ gives the following block structure
	\begin{align*}
		\tG_t^\top\tG_t = \left[
		\begin{array}{c;{2pt/2pt}c}
			(\tG_t^\top\tG_t)_{11} & (\tG_t^\top\tG_t)_{12} \\ \hdashline[2pt/2pt]
			(\tG_t^\top\tG_t)_{21} & (\tG_t^\top\tG_t)_{22}
		\end{array}
		\right] \triangleq \left[
		\begin{array}{c;{2pt/2pt}c}
			\tG_{11}^\top\tG_{11} + \tG_{21}^\top\tG_{21} & \tG_{11}^\top\tG_{12} + \tG_{21}^\top\tG_{22} \\ \hdashline[2pt/2pt]
			\tG_{12}^\top\tG_{11} + \tG_{22}^\top\tG_{21} & \tG_{12}^\top\tG_{12} + \tG_{22}^\top\tG_{22}
		\end{array}
		\right].
	\end{align*}
	Consider the congruence
	\begin{align*}
		&\left[
		\begin{array}{c;{2pt/2pt}c}
			(\tG_t^\top\tG_t)_{11} & (\tG_t^\top\tG_t)_{12} \\ \hdashline[2pt/2pt]
			(\tG_t^\top\tG_t)_{21} & (\tG_t^\top\tG_t)_{22}
		\end{array}
		\right] \sim \\
		&\left[
		\begin{array}{c;{2pt/2pt}c}
			(\tG_t^\top\tG_t)_{11}^{-1/2} & 0 \\ \hdashline[2pt/2pt]
			0 & (\tG_t^\top\tG_t)_{22}^{-1/2}
		\end{array}
		\right] \left[
		\begin{array}{c;{2pt/2pt}c}
			(\tG_t^\top\tG_t)_{11} & (\tG_t^\top\tG_t)_{12} \\ \hdashline[2pt/2pt]
			(\tG_t^\top\tG_t)_{21} & (\tG_t^\top\tG_t)_{22}
		\end{array}
		\right] \left[
		\begin{array}{c;{2pt/2pt}c}
			(\tG_t^\top\tG_t)_{11}^{-1/2} & 0 \\ \hdashline[2pt/2pt]
			0 & (\tG_t^\top\tG_t)_{22}^{-1/2}
		\end{array}
		\right] \\
		&= \left[
		\begin{array}{c;{2pt/2pt}c}
			I & (\tG_t^\top\tG_t)_{11}^{-1/2}(\tG_t^\top\tG_t)_{12}(\tG_t^\top\tG_t)_{22}^{-1/2} \\ \hdashline[2pt/2pt]
			(\tG_t^\top\tG_t)_{22}^{-1/2}(\tG_t^\top\tG_t)_{21}(\tG_t^\top\tG_t)_{11}^{-1/2} & I
		\end{array}
		\right].
	\end{align*}
	Using this relation, one can show that $\tG_t^\top\tG_t\succ 0$; since $(\tG_t^\top\tG_t)_{21} = (\tG_t^\top\tG_t)_{12}^\top$, we only require that  $\|(\tG_t^\top\tG_t)_{11}^{-1/2}(\tG_t^\top\tG_t)_{12}(\tG_t^\top\tG_t)_{22}^{-1/2}\| \le 1$. We assert that this holds with a high probability by our assumptions, for example, by overparameterization and the condition that $|\tG_{22}|\ge |\langle\tG_{21} + \tG_{12},\tilde{v}\rangle|$ for any $\tilde{v}\in\rr^{m+1}$. As a result, we invoke \lemref{thm:mat-prod} and obtain
	\begin{align}
		\norm{\tG_t(\tPhi_t - \tPhi^*)}^2 &= \tr(\tG_t^\top\tG_t(\tPhi_t - \tPhi^*)(\tPhi_t - \tPhi^*)^\top)\nonumber\\
		&\le \tr(\tG_t^\top\tG_t)\lambda_1((\tPhi_t - \tPhi^*)(\tPhi_t - \tPhi^*)^\top)\nonumber\\
		&= \tr(\tG_t^\top\tG_t)\norm{\tPhi_t - \tPhi^*}^2. \label{eq:norm-est2}
	\end{align}
	Using \eqref{eq:norm-est1} and \eqref{eq:norm-est2} in \eqref{eq:norm-est0}, we have
	\begin{align*}
		\norm{\tPhi_{t+1} - \tPhi^*}^2 &\le \norm{\tPhi_t - \tPhi^*}^2 - 2\alpha_t\norm{\tG_{11}^{1/2}\tPhi_1}^2 -2\alpha_t\langle \tG_{21}^\top+\tG_{12},\tPhi_1\rangle - 2\alpha_t\tG_{22}\nonumber \\&\qquad + \alpha_t^2\tr(\tG_t^\top\tG_t)\norm{\tPhi_t - \tPhi^*}^2.
	\end{align*}
	Now, using the inequality $-\abs{\langle \tG_{21}^\top+\tG_{12},\tPhi_1\rangle} \le \langle \tG_{21}^\top+\tG_{12},\tPhi_1\rangle \le \abs{\langle \tG_{21}^\top+\tG_{12},\tPhi_1\rangle}$, we get that
	\begin{align*}
		-\langle \tG_{21}^\top+\tG_{12},\tPhi_1\rangle \le \abs{\langle \tG_{21}^\top+\tG_{12},\tPhi_1\rangle},
	\end{align*}
	and then,
	\begin{align*}
		\norm{\tPhi_{t+1} - \tPhi^*}^2 &\le \norm{\tPhi_t - \tPhi^*}^2 - 2\alpha_t\norm{\tG_{11}^{1/2}\tPhi_1}^2 + 2\alpha_t\abs{\langle \tG_{21}^\top+\tG_{12},\tPhi_1\rangle} - 2\alpha_t\tG_{22}\nonumber \\&\qquad + \alpha_t^2\tr(\tG_t^\top\tG_t)\norm{\tPhi_t - \tPhi^*}^2.
	\end{align*}
	Setting $\tilde{v} = \tPhi_1$ in the condition $|\tG_{22}|\ge |\langle\tG_{21} + \tG_{12}^\top,\tilde{v}\rangle|$, it holds that $|\tG_{22}| > |\langle\tG_{21} + \tG_{12}^\top,\tPhi_1\rangle| - C_1$ for any arbitrary constant $C_1>0$. Set $C_1 = \frac{1}{\alpha_tC_2}$ for some constant $C_2>0$, noting that $\alpha_t > 0$ for all $t$, then we get
	\begin{align*}
		-\tG_{22} \le -\abs{\langle\tG_{21} + \tG_{12}^\top,\tPhi_1\rangle} + \frac{1}{\alpha_tC_2}.
	\end{align*}
	Consequently,
	\begin{align*}
		\norm{\tPhi_{t+1} - \tPhi^*}^2 &\le \norm{\tPhi_t - \tPhi^*}^2 - 2\alpha_t\norm{\tG_{11}^{1/2}\tPhi_1}^2 + \frac{2}{C_2} + \alpha_t^2\tr(\tG_t^\top\tG_t)\norm{\tPhi_t - \tPhi^*}^2.
	\end{align*}
	Now, if the condition $1+M_g\eta_t \le \|\tG_t\|_F$ is such that
	\begin{align*}
		\frac{\norm{\tG_{11}^{1/2}\tPhi_1}^2}{\tr(\tG_t^\top\tG_t)\norm{\tPhi_t - \tPhi^*}^2} \ge \alpha_t \triangleq \frac{\bar{\alpha}_t}{1+M_g\eta_t} \ge \frac{\bar{\alpha}_t}{\norm{\tG_t}_F} \equiv \frac{\bar{\alpha}_t}{\sqrt{\tr(\tG_t^\top\tG_t)}}, 
	\end{align*}
	by fixing $0<\bar{\alpha}_t\equiv\bar{\alpha} < 1$, then
	\begin{align}
		\norm{\tPhi_{t+1} - \tPhi^*}^2 &\le \left(1-\bar{\alpha}\right)\norm{\tPhi_t - \tPhi^*}^2 + \frac{2}{C_2}.\label{eq:recurrence1}
	\end{align}
	The recurrence in \eqref{eq:recurrence1} can be expanded as follows:
	\begin{align*}
		\norm{\tPhi_{t+1} - \tPhi^*}^2 &\le \left(1-\bar{\alpha}\right)\norm{\tPhi_t - \tPhi^*}^2 + \frac{2}{C_2}\\
		&\le \left(1-\bar{\alpha}\right)\left(\left(1-\bar{\alpha}\right)\norm{\tPhi_{t-1} - \tPhi^*}^2 + \frac{2}{C_2}\right) + \frac{2}{C_2}\\
		&\le \left(1-\bar{\alpha}\right)\left(\left(1-\bar{\alpha}\right)\left(\left(1-\bar{\alpha}\right)\norm{\tPhi_{t-2} - \tPhi^*}^2 + \frac{2}{C_2}\right) + \frac{2}{C_2}\right) + \frac{2}{C_2}\\
		&=\left(1-\bar{\alpha}\right)^3\norm{\tPhi_{t-2} - \tPhi^*}^2 + \left(1-\bar{\alpha}\right)^2\frac{2}{C_2} + \left(1-\bar{\alpha}\right)\frac{2}{C_2} + \frac{2}{C_2},
	\end{align*}
	and so on. This gives, for any $T\ge 1$,
	\begin{align}
		\norm{\tPhi_T - \tPhi^*}^2 \le (1-\bar{\alpha})^T\norm{\tPhi_0 - \tPhi^*}^2 + \frac{2}{C_2}\sum_{j=0}^{T-1}(1-\bar{\alpha})^{T-j-1}.\label{eq:recurrence3}
	\end{align}
	Since $C_2>0$ is arbitrary, we set $C_2 = 2\sum_{j=0}^{T-1}(1-\bar{\alpha})^{T-j-1}$. We also have that since $\bar{\alpha} > 0$, it satisfies the inequality $1 - \bar{\alpha} \le \exp(-\bar{\alpha})$. Then \eqref{eq:recurrence3} gives
	\begin{align}
		\norm{\tPhi_T - \tPhi^*}^2 \le \exp(-\bar{\alpha}T)\norm{\tPhi_0 - \tPhi^*}^2 + 1.\label{eq:recurrence2}
	\end{align}
	Substituting our choice of $T$ into \eqref{eq:recurrence2} gives
	\begin{align*}
		\norm{\tPhi_T - \tPhi^*}^2 \le \epsilon + 1,
	\end{align*}
	which is result \ref{p:result1}.
	
	To prove \ref{p:result2}, we first notice that the local condition $\norm{\hessn_\Phi{\hat{R}_s}(\Phi)}_{op}\le D_R$ in \ref{ass:r1} implies local $D_R$-Lipschitz continuity of $\gradn_\Phi{\hat{R}_s(\Phi)}$ with respect to $\Phi$, that is, for $\bar{\theta},\tilde{\theta}$ around the initialization, we have
	\begin{align*}
		\norm{\gradn{\hat{R}_s(\Phi(\cdot,\bar{\theta}))} - \gradn{\hat{R}_s(\Phi(\cdot,\tilde{\theta}))}} \le D_R\norm{\Phi(\cdot,\bar{\theta}) - \Phi(\cdot,\tilde{\theta})},
	\end{align*}
	or equivalently,
	\begin{align}
		\hat{R}_s(\Phi(\cdot,\bar{\theta})) \le \hat{R}_s(\Phi(\cdot,\tilde{\theta})) + \langle \gradn{\hat{R}_s(\Phi(\cdot,\tilde{\theta}))}, \Phi(\cdot,\bar{\theta}) - \Phi(\cdot,\tilde{\theta})\rangle + \frac{D_R}{2}\norm{\Phi(\cdot,\bar{\theta}) - \Phi(\cdot,\tilde{\theta})}^2.\label{eq:R-Lipschitz}
	\end{align}
	We recall the notation $\Phi_t \triangleq \Phi(\cdot,\theta_t)$ for all $t\in\rnn$. Then, using $\calL \triangleq \hat{R}_s + g$, \lemref{thm:g-bound0}, and \eqref{eq:R-Lipschitz}, we get
	\begin{align*}
		\calL(\theta_{t+1}) &\le \calL(\theta_t) + \langle \gradn{\hat{R}_s(\Phi_t)}, \Phi_{t+1} - \Phi_t\rangle + \frac{D_R}{2}\norm{\Phi_{t+1} - \Phi_t}^2 + \langle \gradn{g(\theta_t)}, \theta_{t+1} - \theta_t\rangle\\ &\quad + \omega_\nu(d_\nu(\theta_t, \theta_{t+1}))\norm{\theta_{t+1} - \theta_t}_{\theta_t}^2\\
		&\le \calL(\theta_t) - \frac{\gamma_R}{2}\norm{\Phi_{t+1} - \Phi_t}^2 - \calL(\theta_t) + \calL(\theta_{t+1})  - \omega_\nu(-d_\nu(\theta_t, \theta_{t+1}))\norm{\theta_{t+1} - \theta_t}_{\theta_t}^2\\ &\quad + \frac{D_R}{2}\norm{\Phi_{t+1} - \Phi_t}^2 + \omega_\nu(d_\nu(\theta_t, \theta_{t+1}))\norm{\theta_{t+1} - \theta_t}_{\theta_t}^2.
	\end{align*}
	Using the $\gamma_R$-strong convexity assumption on $\hat{R}$ in \ref{ass:r1} and the Lipschitz property of $\calL$ in \lemref{thm:L-Lipschitz}, this gives
	\begin{align}
		\calL(\theta_{t+1})	&\le \calL(\theta_t) + (B_RB_\Phi + B_g)\norm{\theta_{t+1} - \theta_t} + \frac{D_R-\gamma_R}{2}\norm{\Phi_{t+1} - \Phi_t}^2\nonumber\\ &\quad + \left(\omega_\nu(d_\nu(\theta_t, \theta_{t+1})) - \omega_\nu(-d_\nu(\theta_t, \theta_{t+1}))\right)\norm{\theta_{t+1} - \theta_t}_{\theta_t}^2\nonumber\\
		&\le \calL(\theta_t) + (B_RB_\Phi + B_g)\norm{\theta_{t+1} - \theta_t} + \frac{B_\Phi^2(D_R-\gamma_R)}{2}\norm{\theta_{t+1} - \theta_t}^2\nonumber\\ &\quad + \left(\omega_\nu(d_\nu(\theta_t, \theta_{t+1})) - \omega_\nu(-d_\nu(\theta_t, \theta_{t+1}))\right)\norm{\theta_{t+1} - \theta_t}_{\theta_t}^2. \label{eq:l-est0}
	\end{align}
	Recalling the notation $\delta_t \triangleq \theta_{t+1} - \theta_t$ and substituting the estimates on $\norm{\delta_t}$ and $\norm{\delta_t}_{\theta_t}$ from \lemref{thm:step-dist} into \eqref{eq:l-est0} yields result \ref{p:result2}.
\end{proof}
\section{Additional experimental details and results}\label{app:add-mnist}
\subsection{Remark on the T-I measure}\label{app:ti-measure}
The time-invariance measure provides a way to measure stability of the optimizer's dynamics from initialization. However, since the signum function does not account for indices $(i,j)$ of $a_l$ with $a_{ij} = 0$, \ie,
\begin{align*}
	\sign(a_{ij}) \triangleq
	\begin{cases}
		+1 & \text{if }a_{ij} > 0,\\
		-1 & \text{if }a_{ij} < 0,
	\end{cases}
\end{align*}
and, as we have seen, the GGN-SCORE framework potentially produces many of this instance (with $a_{ij} = 0$) to reduce the model's complexity and/or improve generalization, a natural question is what state should be assumed for neuron $a_{ij}$ when it is exactly zero. For this, we follow the standard convention that if $a_{ij} = 0$, then the $(i,j)$-th neuron remains unchanged from its initial state \cite[Section 13.7]{haykin2009neural}. Under this convention, the proportion of the indices $(i,j)$ of $a_l^{\text{final}}$ satisfying $\sign(a_{ij}^{\text{start}}) \neq \sign(a_{ij}^{\text{final}})$ with $a_{ij}^{\text{final}} = 0$ contribute to the stability of activations, and hence should be accounted for in the T-I measure. However, this contribution appear to be insignificant for the values of $\tau$ and $\mu$ that give the best test accuracies. From what we observe in \figref{fig:testloss-vs-tau-mu-mnist} and \figref{fig:accuracy-mnist}, proper choices of $\mu$ and $\tau$ reliably produces stable dynamics of the optimizer as well as a good generalization of the final trained model.
\subsection{MNIST teacher-student setting}
In order to evaluate GGN-SCORE on the MNIST dataset such that we are close to the theoretical framework, we consider a teacher-student setup for the MNIST dataset in a similar way as \cite[Appendix C.4]{arbel2023rethinking}:
\begin{figure}[h!]
	\centering
	\includegraphics[width=0.8\textwidth]{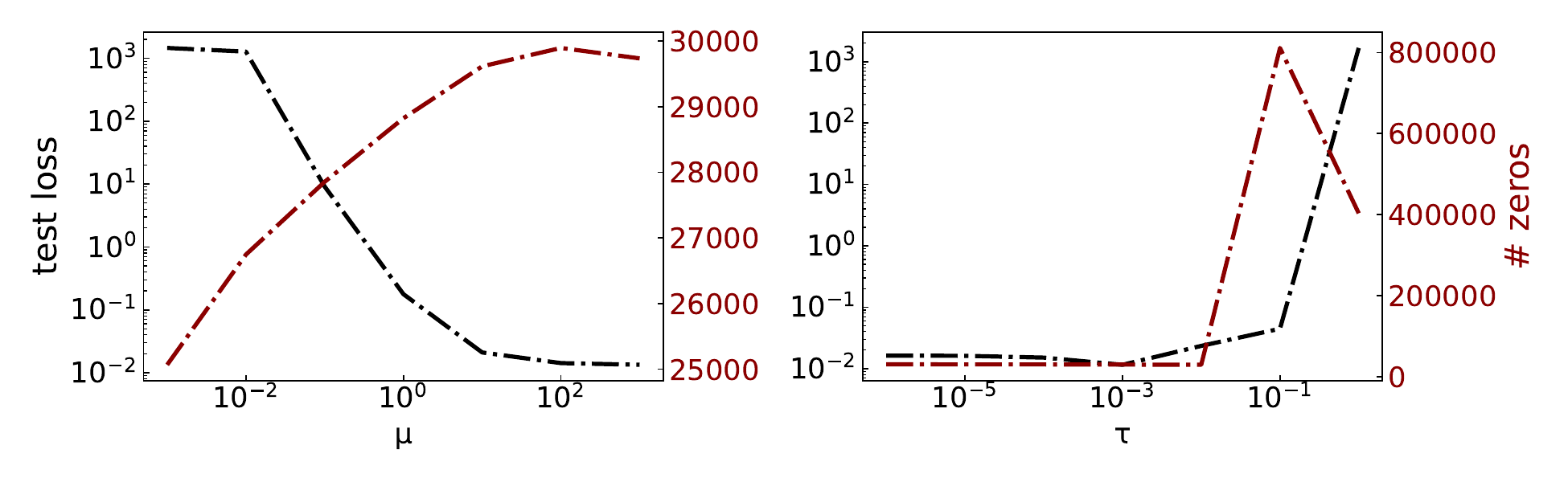}
	\caption{Test loss of the GGN-SCORE-trained NN on MNIST dataset (teacher-student) with $g(\theta)$ given by \eqref{eq:g-example}. \textbf{Left:} Results for different values of the regularization smoothing parameter $\mu$ with $\tau=10^{-4}$ fixed. \textbf{Right:} Results for different values of the regularization strength $\tau$ with $\mu=1/\sqrt{n}$ fixed.}
	\label{fig:testloss-vs-tau-mu-mnist-t-s}
\end{figure}
\begin{figure*}[h!]
	\centering
	\begin{subfigure}[b]{0.3\textwidth}
		\centering
		\includegraphics[width=\textwidth]{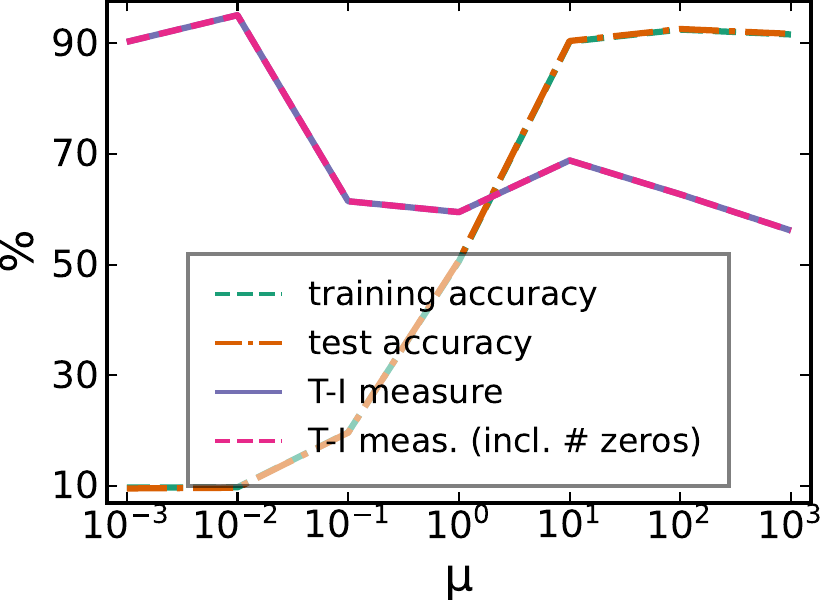}
	\end{subfigure}
	\hfil
	\begin{subfigure}[b]{0.3\textwidth}
		\centering
		\includegraphics[width=\textwidth]{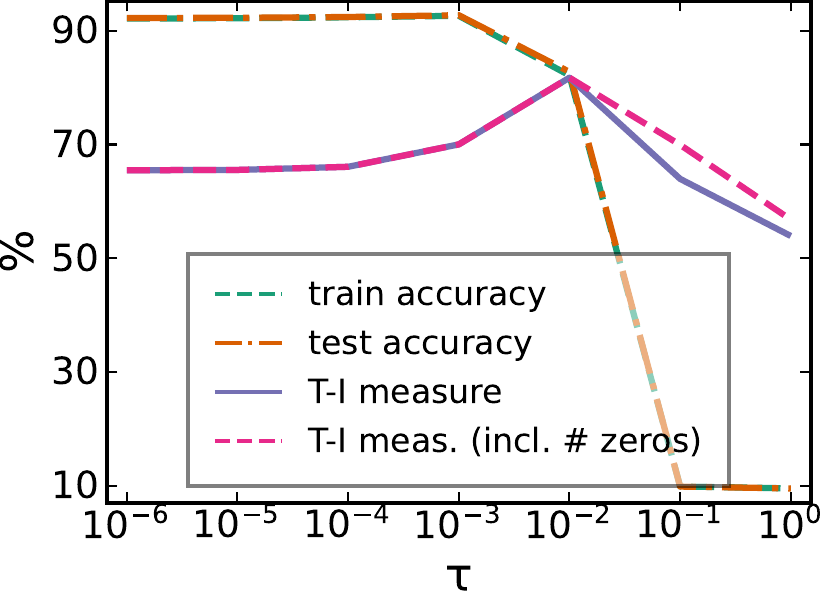}
	\end{subfigure}
	\caption{Accuracy and T-I measure of the GGN-SCORE-trained NN on MNIST dataset (teacher-student) for different values of $\mu$ (\textbf{left}) and different values of $\tau$ (\textbf{right}), with the regularization function $g(\theta)$ given by \eqref{eq:g-example}. In the left figure, $\tau=10^{-4}$ is used. In the right figure, $\mu=1/\sqrt{n}$ is used.}
	\label{fig:accuracy-mnist-t-s}
\end{figure*}
\begin{itemize}
	\item We create a custom training dataset by combining the original MNIST test dataset (containing $10000$ sample points) and a balanced subset of the original training dataset. This balanced subset is created by ``undersampling" the first $3000$ samples of the original training dataset to give $2610$ sample points. In total, the custom training dataset contains $12610$ sample points.
	\item We then train a teacher NN $\Phi^*$ of the form \eqref{eq:teacher-nn} and hidden size $n^*=16$ on this training dataset with the cross-entropy loss function and the SiLU activation function.
	\item A training ``target" dataset is created from $\Phi^*$ (with the softmax function applied on each output of $\Phi^*$).
	\item The student NN of the form \eqref{eq:nn} with the SiLU activation and hidden size $n=1024$ is then trained on the custom training input samples and their corresponding target samples constructed from $\Phi^*$. The trained student NN is tested on the original MNIST test dataset.
\end{itemize}

The training and test results are displayed \figref{fig:testloss-vs-tau-mu-mnist-t-s} and \figref{fig:accuracy-mnist-t-s}. We follow a similar evaluation procedure as in Section~\ref{sec:mnist}, \ie, the results are evaluated on the basis of the test loss, training and test accuracy, and T-I measure of the trained student NN. Interestingly, similar observations as in Section~\ref{sec:mnist} are made from the displayed results. The total computation time to generate the results in \figref{fig:testloss-vs-tau-mu-mnist-t-s} and \figref{fig:accuracy-mnist-t-s} is $\sim8$ hours, $39$ minutes on CPU.

\subsection{FashionMNIST experiments}
We perform experiments on the FashionMNIST dataset \cite{fashionmnist} under the same setting as the MNIST experiments in Section~\ref{sec:mnist}. While the FashionMNIST classification tends to be a harder task than the MNIST, results shown in \figref{fig:testloss-vs-tau-mu-fashionmnist} and \figref{fig:accuracy-fashionmnist} indicate similar behaviours as those described in Section~\ref{sec:mnist} regarding the influence of the regularization parameters.
\subsection{Comparison with GD}
We now compare GGN-SCORE with GD on three UCI benchmark datasets\footnote{\url{https://archive.ics.uci.edu}.}: pendigits, letter, and avila, summarized in Table~\ref{tab:datasets}. As in Section~\ref{ss:experiments}, we use a learning rate of $1$ for GD, and set the hidden size $n=128$ in all the experiments for a NN of the form \eqref{eq:nn}, and a scaling $\kappa(n)=1/\sqrt{n}$. The function $g$ in GGN-SCORE is given by \eqref{eq:g-example} with $\tau=10^{-4}$. The results are shown \figref{fig:results-uci} and Table~\ref{tab:stability}. We observe faster convergence and better generalization in most cases for GGN-SCORE, and as in the case for the full-batch deterministic setting in Section~\ref{ss:experiments} on synthetic datasets, we achieve this performance in faster time compared to GD. Note that much of the computational burden associated with the regularized GGN is greatly reduced by using the stylized expression \eqref{eq:ggn-step-approx}, since the mini-batch size is typically much smaller than $p$, the size of the optimization variable $\theta$.
\begin{figure}[h!]
	\centering
	\includegraphics[width=0.8\textwidth]{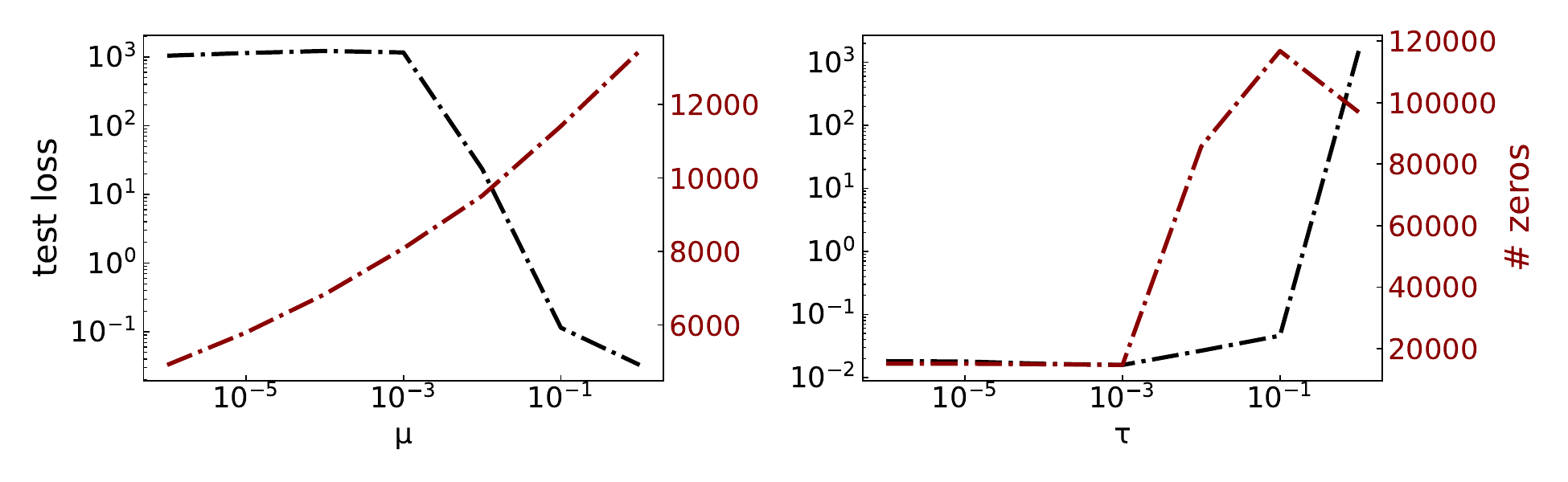}
	\caption{Test loss evaluation of the GGN-SCORE-trained NN on FashionMNIST dataset for different values of the regularization smoothing parameter $\mu$ fixing $\tau=10^{-4}$ (\textbf{left}) and different values of the regularization strength $\tau$ fixing $\mu=1/\sqrt{n}$ (\textbf{right}), where the regularization function $g(\theta)$ is given by \eqref{eq:g-example}.}
	\label{fig:testloss-vs-tau-mu-fashionmnist}
\end{figure}
\begin{figure*}[t!]
	\centering
	\begin{subfigure}[b]{0.3\textwidth}
		\centering
		\includegraphics[width=\textwidth]{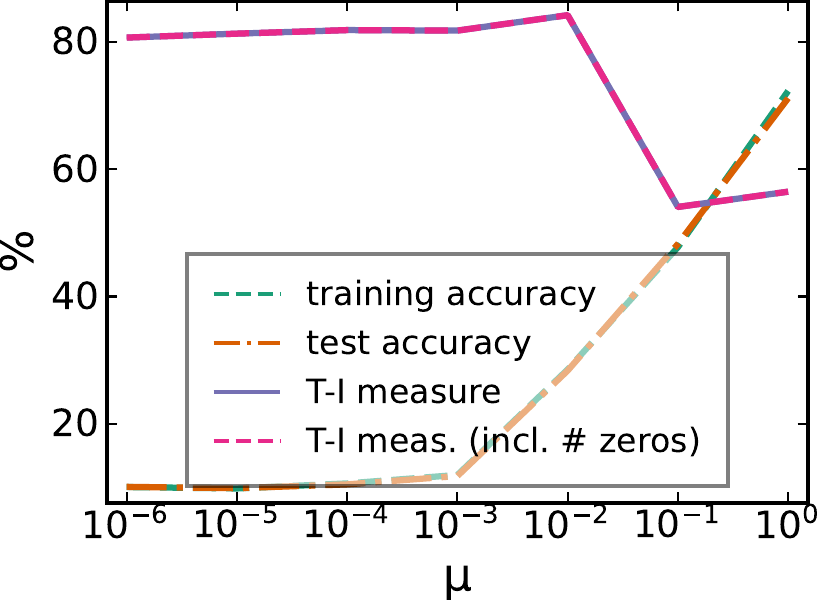}
	\end{subfigure}
	\hfil
	\begin{subfigure}[b]{0.3\textwidth}
		\centering
		\includegraphics[width=\textwidth]{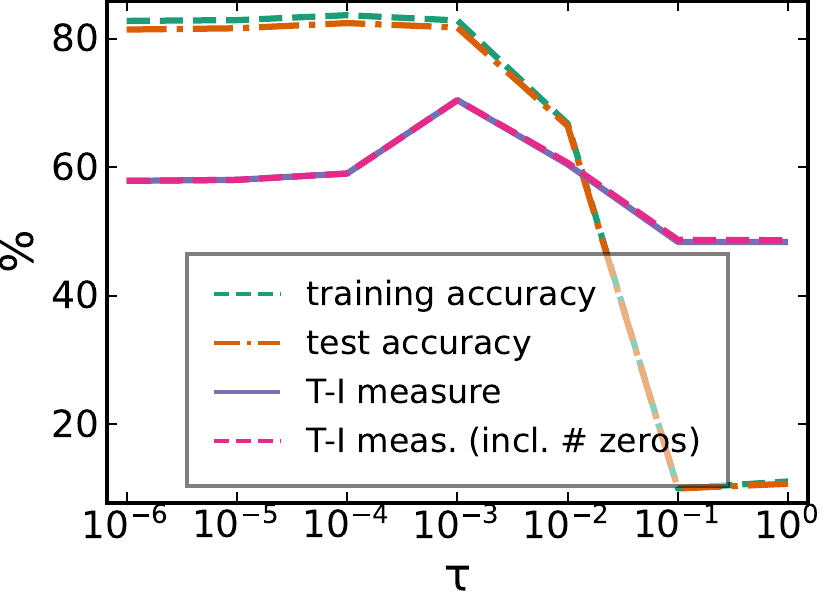}
	\end{subfigure}
	\caption{Evaluation of GGN-SCORE on FashionMNIST dataset for different values of the regularization strength $\mu$ (\textbf{left}) and different values of the regularization smoothing parameter $\tau$ (\textbf{right}). The regularization function $g(\theta)$ is given by \eqref{eq:g-example}. In the left figure, $\mu=1/\sqrt{n}$ is used. In the right figure, $\tau=10^{-4}$ is used.}
	\label{fig:accuracy-fashionmnist}
\end{figure*}
\begin{figure*}[h!]
	\centering
	\begin{subfigure}[b]{\textwidth}
		\centering
		\includegraphics[width=0.9\textwidth]{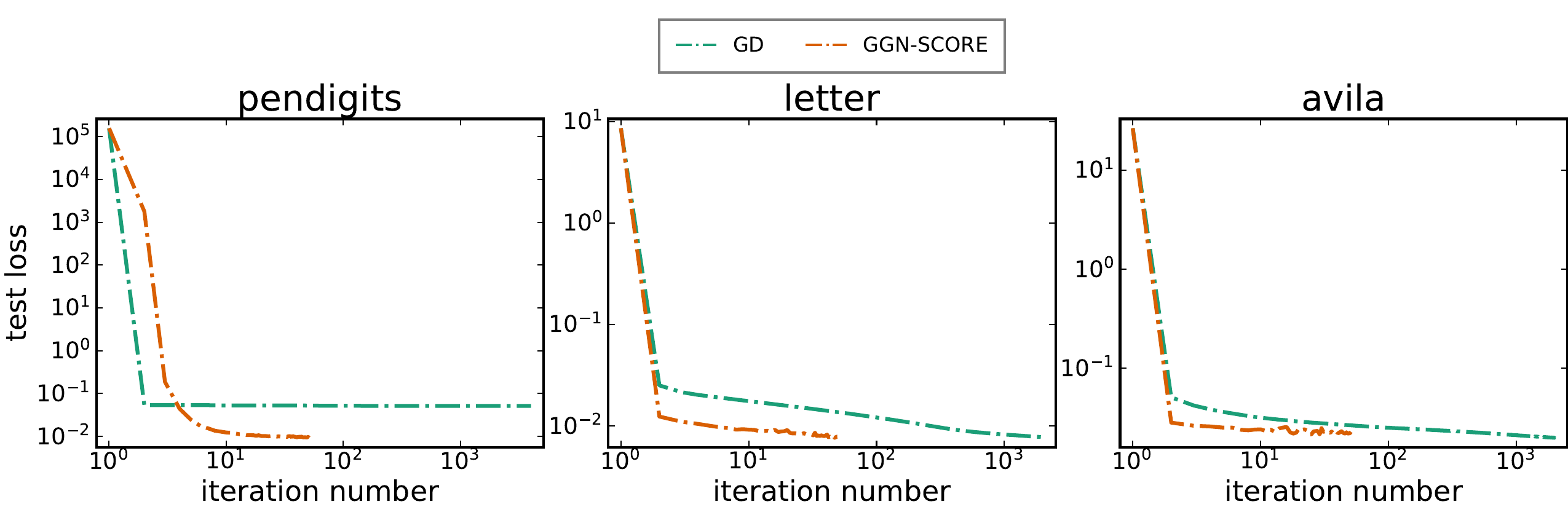}
	\end{subfigure}
	\vfil
	\begin{subfigure}[b]{\textwidth}
		\centering
		\includegraphics[width=0.9\textwidth]{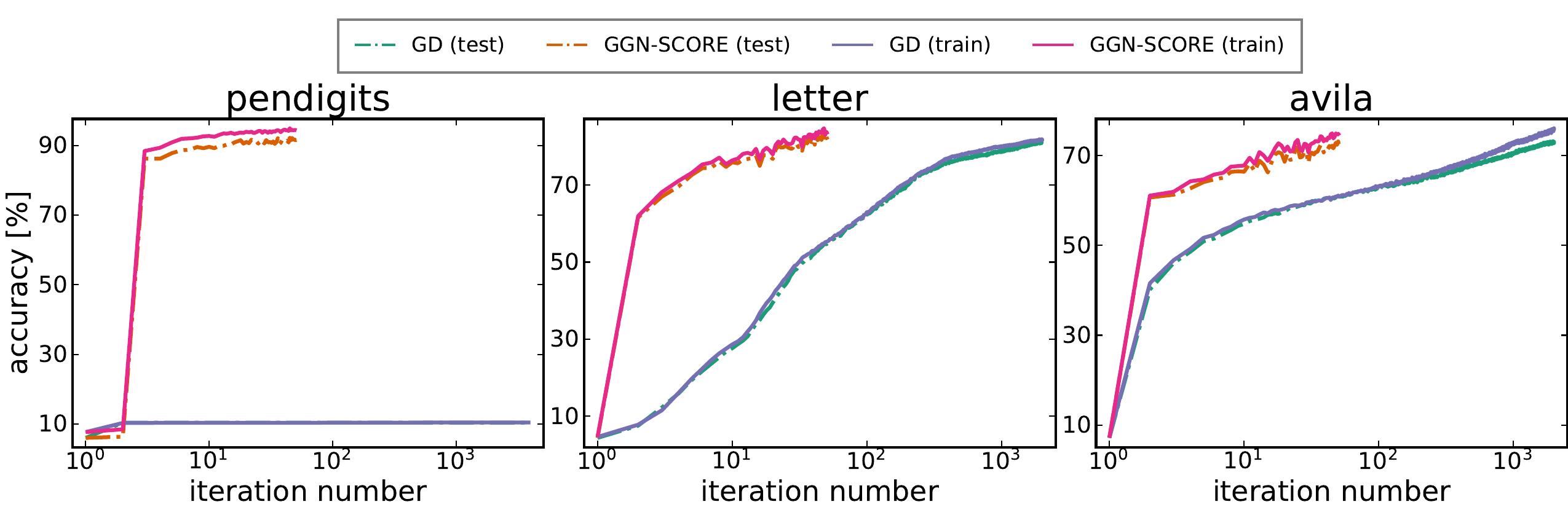}
	\end{subfigure}
	\vfil
	\begin{subfigure}[b]{\textwidth}
		\centering
		\includegraphics[width=0.9\textwidth]{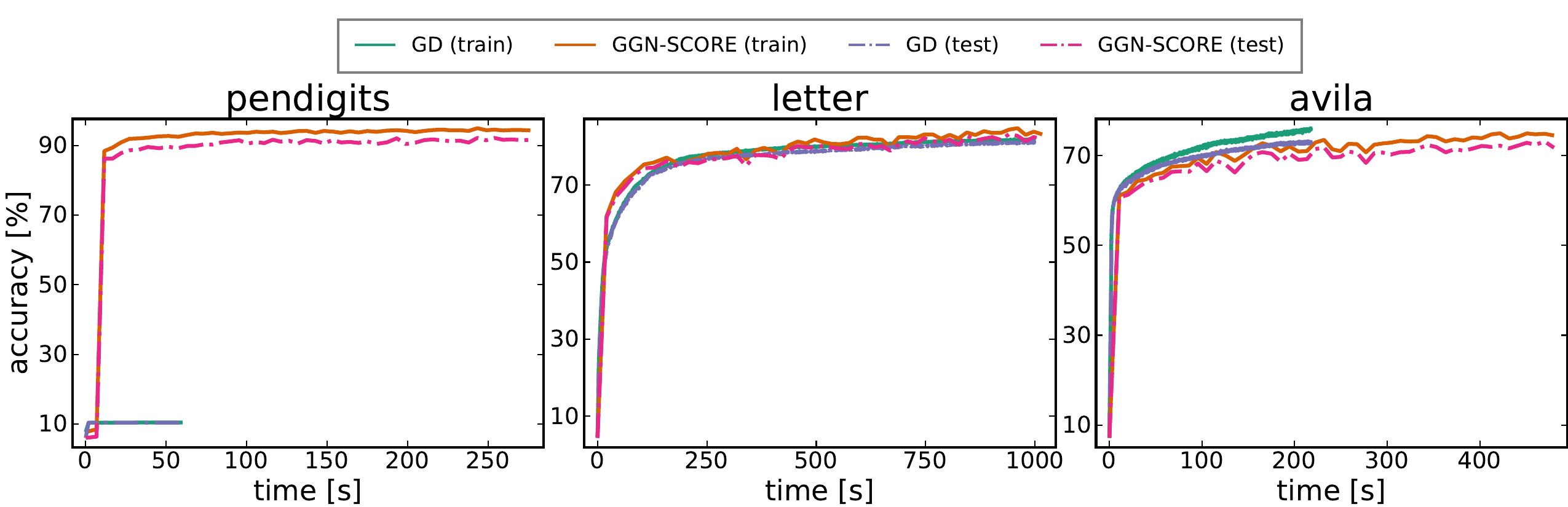}
	\end{subfigure}
	\caption{Test loss and accuracy evolution for GD and GGN-SCORE on pendigits, letter and avila datasets. The regularization function $g(\theta)$ in GGN-SCORE is given by \eqref{eq:g-example} with $\tau=10^{-4}$ and $\mu$ given in Table~\ref{tab:stability}.}
	\label{fig:results-uci}
\end{figure*}
\begin{table}[h!]
	\caption{Summary of UCI datasets used for comparison.}
	\label{tab:datasets}
	\centering
	\begin{tabular}{lcccc}
		\toprule
		&\multicolumn{2}{c}{Num. of samples}    \\
		\cmidrule(r){2-3}
		Dataset & Training & Test & Input dim. & Num. of classes  \\
		\midrule
		pendigits & 7494  & 3498  & 16 & 10   \\
		letter     & 10500 & 5000  & 16 & 26   \\
		avila     & 10430   & 10437 & 11 & 12   \\
		\bottomrule
	\end{tabular}
\end{table}

\begin{table}[h!]
	\caption{Stability of activation measure.}
	\label{tab:stability}
	\centering
	\begin{tabular}{lcccccc}
		\toprule
		&&&\multicolumn{2}{c}{T-I measure ($\%$)}  & \multicolumn{2}{c}{T-I meas. incl. $a_{ij}=0$ ($\%$)}    \\
		\cmidrule(r){4-7}
		Dataset & Batch-size & $\mu$ & GD & GGN-SCORE & GD & GGN-SCORE  \\
		\midrule
		pendigits & 8  & $0.001/\sqrt{n}$  & 50.066 & 52.7331 & 50.1041 & 52.7331   \\
		letter     & 64 & $10/\sqrt{n}$  & 55.9237 & 55.2016 & 55.925 &  55.2112   \\
		avila     & 64   & $10/\sqrt{n}$ & 73.318 & 71.4815 & 73.3189 & 71.4833   \\
		\bottomrule
	\end{tabular}
\end{table}

\end{document}